%% file: main.tex
\documentclass[10pt,onecolumn,letterpaper]{article}

\usepackage[pagenumbers]{cvpr} 

\usepackage{algorithm}
\usepackage{algorithmicx}
\usepackage{algpseudocode}

\input{preamble}

\definecolor{cvprblue}{rgb}{0.21,0.49,0.74}
\usepackage[breaklinks,colorlinks,allcolors=cvprblue]{hyperref}

\def\paperID{4395} 
\def\confName{CVPR}
\def\confYear{2026}

\title{DSP-Reg: Domain-Sensitive Parameter Regularization for Robust Domain Generalization}

\author{Xudong Han${^{1,*}}$, Senkang Hu$^{2,3,*,\dagger}$ , Yihang Tao$^{2,3}$, Yu Guo$^{2,3}$, Philip Birch$^{1}$, Sam Tak Wu Kwong$^{4}$, \\Yuguang Fang$^{2,3}$\\
${^1}$University of Sussex, ${^2}$Hong Kong JC STEM Lab of Smart City, ${^3}$CITYUHK, $^{4}$Lingnan University.\\
{\tt\small \{senkang.forest, yihang.tommy, yuguo46-c\}@my.cityu.edu.hk, \{xh218, P.M.Birch\}@sussex.ac.uk,}\\
{\tt\small samkwong@ln.edu.hk,  my.Fang@cityu.edu.hk}\\
{$^*$Equal contribution, $^\dagger$ Project leader}
}

\begin{document}
\maketitle

\tableofcontents
\input{sections/0_abstract}

\input{sections/intro.tex}

\input{sections/background_related_work.tex}

\input{sections/problem_setup.tex}
\input{sections/method.tex}

\input{sections/experiments.tex}
\input{sections/conclusion_ackno.tex}

{
    \small
    \bibliographystyle{ieeenat_fullname}
    \bibliography{ref,ref2}
}

\clearpage
\newpage
\onecolumn

\input{sections/Appendix.tex}

\end{document}

%% file: preamble.tex
\usepackage{amsthm}
\theoremstyle{plain}

\newtheorem{lemma}{Lemma}
\newtheorem{proposition}{Proposition}
\newtheorem{theorem}{Theorem}
\theoremstyle{definition}
\newtheorem{definition}{Definition}
\theoremstyle{remark}

\newcommand{\E}{\mathbb{E}}
\newcommand{\Cov}{\operatorname{Cov}}
\newcommand{\Var}{\operatorname{Var}}
\newcommand{\R}{\mathbb{R}}
\newcommand{\diag}{\operatorname{diag}}

%% file: sections/0_abstract.tex
\begin{abstract}
    Domain Generalization (DG) is a critical area that focuses on developing models capable of performing well on data from unseen distributions, which is essential for real-world applications.
    Existing approaches primarily concentrate on learning domain-invariant features, which assume that a model robust to variations in the source domains will generalize well to unseen target domains. However, these approaches neglect a deeper analysis at the parameter level, which  makes the model hard to explicitly differentiate between parameters sensitive to domain shifts and those  robust, potentially hindering its overall ability to generalize.
    In order to address these limitations, we first build a covariance-based parameter sensitivity analysis framework to quantify the sensitivity of each parameter in a model to domain shifts. By computing the covariance of parameter gradients across multiple source domains, we can identify parameters that are more susceptible to domain variations, which serves as our theoretical foundation. 
    Based on this, we propose Domain-Sensitive Parameter Regularization (DSP-Reg), a principled framework that guides model optimization by a soft regularization technique that encourages the model to rely more on domain-invariant parameters while suppressing those that are domain-specific. This approach provides a more granular control over the model's learning process, leading to improved robustness and generalization to unseen domains.
    Extensive experiments on benchmarks, such as PACS, VLCS, OfficeHome, and DomainNet, demonstrate that DSP-Reg outperforms state-of-the-art approaches, achieving an average accuracy of 66.7\% and surpassing all baselines.
    
\end{abstract}

%% file: sections/intro.tex
\section{Introduction}
\label{sec:intro}
Domain Generalization (DG) remains a fundamental yet challenging problem in machine learning, addressing the critical task of training models that maintain robust predictive capabilities across unseen, out-of-distribution target domains. This capability is crucial for real-world applications such as autonomous driving \cite{huFullSceneDomainGeneralization2024,huAdaptiveCommunicationsCollaborative2024,huCollaborativePerceptionConnected2024,huAgentsCoMergeLargeLanguage2025,huAgentsCoDriverLargeLanguage2024,huFullSceneDomainGeneralization2024}, medical imaging \cite{liu2020shape}, face recognition \cite{wen2016discriminative}, and AIGC \cite{guo2025neptune-x,li2025instruct2see,guo2024onerestore, hu2025distributionaligneddecodingefficientllm}, where even slight changes in environments can induce substantial distributional shifts. In addition, standard models typically suffer significant performance degradation under such shifts, underscoring limitations to their reliability in practice.

Currently, most DG methods focus on learning domain-invariant feature representations from multiple source domains \cite{huFullSceneDomainGeneralization2024,huAdaptiveCommunicationsCollaborative2024,yang_fda_2020,zhouDomainGeneralizationSurvey2022,nguyenDomainInvariantRepresentation2022}. Representative strategies include domain-adversarial training to promote feature invariance \cite{10.5555/2946645.2946704}, reducing distributional discrepancies using metrics such as Maximum Mean Discrepancy \cite{li2018domain}, and enforcing invariant predictive relationships as in Invariant Risk Minimization (IRM) \cite{arjovsky2019invariant}. Although these methods have achieved notable progress, they primarily operate at the feature level and lack direct, fine-grained control over individual model parameters. As a result, these methods are unable to explicitly differentiate between domain-invariant and domain-sensitive components, which may limit its overall generalization capability.

To address this limitation, recent research has begun examining model parameters directly through sensitivity analysis. For example, Fisher Information matrix \cite{kirkpatrick2017overcoming} provides valuable insights into how model parameters respond differently under domain shifts. Approaches like Fisher regularization (Fishr) \cite{rame2022fishr} utilize gradient-based analyses to enforce invariance in gradient distributions across domains, thus partially addressing parameter-level domain sensitivity. Nevertheless, these approaches remain limited because of their reliance on static, local curvature estimates and are often ineffective under dynamically evolving or non-stationary domain shifts.

Motivated by the observation that existing approaches underutilize parameter-level information, we begin our investigation from the perspective of \emph{model parameter sensitivity}. 
Specifically, we propose a covariance-based parameter sensitivity analysis framework to quantify the sensitivity of each parameter to domain shifts. First, we introduce a linear perturbation model to locally approximate the change in output resulting from small changes in the input and parameters. Then, by applying the law of covariance propagation, we analyze how these uncertainties propagate through the network and affect the final output variance. Finally, we derive a single-parameter sensitivity index $s_k$ to measure the contribution of the variance of parameter $\theta_k$ to the variance of the final output; a larger $s_k$ indicates that small fluctuations in $\theta_k$ significantly affect the output, making $\theta_k$ \emph{sensitive} to inputs. Building on this theoretical foundation, we propose a novel approach named Domain-Sensitive Parameter Regularization (DSP-Reg). The key idea is that \emph{if a parameter exhibits significant sensitivity differences across different source domains, it is likely to have learned domain-specific features rather than generalizable knowledge}. We then design a soft regularization scheme that penalizes domain-sensitive parameters, encouraging the model to rely on domain-invariant parameters that perform consistently across all domains.
Our method's granular control enhances model robustness and generalization, potentially extending to other DG strategies as a complementary mechanism.
In summary, the primary contributions of this work are as
follows.
\begin{itemize}[nosep]
    \item Introducing a principled covariance-based parameter sensitivity framework to systematically quantify parameter sensitivity to domain shifts.
    \item Proposing DSP-Reg, which can adaptively penalize domain-sensitive parameters, substantially improving model generalization.
    \item Demonstrating through extensive experiments on standard DG benchmarks (PACS, VLCS, OfficeHome, TerraIncognita, DomainNet) that DSP-Reg significantly outperforms state-of-the-art approaches.
\end{itemize}

%% file: sections/background_related_work.tex
\section{Background and Related Work}
\label{sec:background_related_work}

\subsection{Domain Generalization}
Domain Generalization (DG) aims to train models capable of generalizing to unseen, out-of-distribution target domains by learning from one or multiple source datasets, which may exhibit differing data distributions. DG methods can be broadly distinguished based on the nature of the source data: single-source approaches typically operate without leveraging information about distinct domains within the training set, effectively treating all source data as a single unit, whereas multi-source algorithms explicitly utilize domain labels and exploit the statistical differences between the available source domain distributions to inform the learning process. A predominant strategy involves focusing on domain-invariant representations, where models strive to capture underlying data characteristics that are robust to domain shifts. This has been pursued through various means, including adversarial training \cite{ganin2016domain}, distribution discrepancy minimization using metrics like Maximum Mean Discrepancy (MMD) \cite{li2018domain,long2015learning, muandet2013domain}, and invariance principles such as Invariant Risk Minimization (IRM) \cite{arjovsky2019invariant,krueger2021out}. Other common approaches include extensive data augmentation \cite{carlucci2019domain, zhou2021mixstyle} to expose models to wider variations, and meta-learning \cite{balaji2018metareg,dou2019domain,li2018learning}, which trains models to adapt to new domains rapidly. Additionally, domain alignment \cite{ganin2016domain,muandet2013domain} and self-supervised learning \cite{he2020momentum} have also been proposed to focus on learning domain-invariant features by minimizing discrepancies between source domains. Recently, FisherTune \citep{zhao2025fishertune} has leveraged Vision Foundation Models (VFMs) for DG tasks, particularly in semantic segmentation, capitalizing on their robust pre-trained representations.

\subsection{Parameter Sensitivity Analysis}
Parameter Sensitivity Analysis in DG focuses on understanding how model parameters respond to domain shifts and using this insight to enhance generalization. By identifying parameters sensitive to domain changes, models can be designed to prioritize domain-invariant features. Fisher Information-based metrics, introduced by Elastic Weight Consolidation (EWC) \cite{kirkpatrick2017overcoming}, provide insights into parameter importance by quantifying sensitivity to loss perturbations. Such analyses reveal the normalisation layers, particularly BatchNorm parameters, as highly sensitive to domain shifts \cite{wang2019easy}. Strategies addressing parameter sensitivity include adaptive recalibration methods (AdaBN) and targeted regularization techniques like DomainDrop~\cite{guo2023domaindrop}, which selectively masks domain-specific neurons. Additionally, gradient alignment techniques \cite{rame2022fishr} ensure that parameter updates are consistent across domains, reducing sensitivity and enhancing robustness. FisherTune \cite{zhao2025fishertune} leverages parameter sensitivity metrics to selectively fine-tune vision foundation models, achieving improved robustness to domain shifts without compromising pretrained knowledge.

%% file: sections/problem_setup.tex
\section{Covariance-Based Parameter Sensitivity Analysis}\label{sec:sensitivity-cov}

As stated in the introduction, our goal is to build a robust domain generalization framework from the perspective of model parameter sensitivity, and our key idea is that \emph{if a parameter exhibits significant sensitivity differences across different source domains, it is likely to have learned domain-specific features rather than generalizable knowledge}. The key point here is how to quantify the sensitivity of a parameter to domain shifts. In this section, we present a covariance-based parameter sensitivity analysis framework. It includes three main steps: 1) First, we formulate a linearized perturbation model that locally approximates the mapping from parameter perturbations to output variations. 2) Second, we apply the law of propagation of second-order moments to analyze how parameter uncertainty (variance) propagates through the network and ultimately impacts the stability of the output. 3) Finally, based on this analysis, we derive a clear single-parameter sensitivity index ($s_k$) and reveal its deep connection to the classical Fisher Information, providing a unified perspective for the entire framework.

\subsection{Linearised Perturbation Model}
\label{sec:linear-perturb}

We begin by relating domain-induced variability to a local, first-order model of how input and parameters jointly perturb the network output. Formally, consider a deep neural network (DNN)
$f : \mathbb{R}^{d_x} \times \mathbb{R}^{d_\theta} \;\longrightarrow\; \mathbb{R}^{d_y}$,
with parameters $\mathbf{\theta}\in\mathbb{R}^{d_\theta}$ evaluated at a nominal operating point $(\mathbf{x}_0,\mathbf{\theta}_0)$.  Let
\begin{align}
J_x &\;:=\; \left.\frac{\partial f_{\mathbf{\theta}}}{\partial \mathbf{x}}\right|_{(\mathbf{x}_0,\mathbf{\theta}_0)} \in \mathbb{R}^{d_y\times d_x},\\
J_{\mathbf{\theta}} &\;:=\; \left.\frac{\partial f_{\mathbf{\theta}}}{\partial \mathbf{\theta}}\right|_{(\mathbf{x}_0,\mathbf{\theta}_0)} \in \mathbb{R}^{d_y\times d_\theta}
\end{align}
denote the Jacobians with respect to the input and parameters, respectively.

For infinitesimal perturbations $(\delta_\mathbf{x},\delta_{\mathbf{\theta}})$, a first-order Taylor expansion yields
\begin{equation}\label{eq:lin-perturb}
\delta_\mathbf{y} \approx J_x\,\delta_\mathbf{x} + J_{\mathbf{\theta}}\,\delta_{\mathbf{\theta}}.
\end{equation}

\paragraph{Derivation.}  Write the perturbed output explicitly as
\begin{equation}
\begin{aligned}
\delta_\mathbf{y}
&:= f_{\mathbf{\theta}_0+\delta_{\mathbf{\theta}}}(\mathbf{x}_0+\delta_\mathbf{x}) - f_{\mathbf{\theta}_0}(\mathbf{x}_0).\\[4pt]
&= f_{\mathbf{\theta}_0}(\mathbf{x}_0)
  + \left.\frac{\partial f}{\partial \mathbf{x}}\right|_{(\mathbf{x}_0,\mathbf{\theta}_0)}\delta_\mathbf{x}
  + \left.\frac{\partial f}{\partial \mathbf{\theta}}\right|_{(\mathbf{x}_0,\mathbf{\theta}_0)}\delta_{\mathbf{\theta}}\\[-2pt]
  &+ \underbrace{\mathcal{O}\Bigl(\lVert\delta_\mathbf{x}\rVert^2 + \lVert\delta_{\mathbf{\theta}}\rVert^2\Bigr)}_{\text{higher order}} - f_{\mathbf{\theta}_0}(\mathbf{x}_0).\\[4pt]
&\approx J_\mathbf{x}\,\delta_\mathbf{x} + J_\mathbf{\theta}\,\delta_{\mathbf{\theta}},
\end{aligned}
\end{equation}
where $J_x$ and $J_{\mathbf{\theta}}$ are the input and parameter Jacobians defined above.  The approximation discards second and higher-order terms for \emph{infinitesimal} perturbations.

\subsection{Propagation of Second-Order Moments}
\label{sec:cov-propagation}
To separate the contributions of input and parameter variability to output stability, we propagate second-order moments through the linear relation derived in Section~\ref{sec:linear-perturb}. 
\begin{theorem}
Assuming the perturbations are centered and independent,
$\delta_\mathbf{x}\sim\mathcal N(\mathbf 0,\,\mathbf{\Sigma}_\mathbf{x})$,
$\delta_{\mathbf{\theta}}\sim\mathcal N(\mathbf 0,\,\mathbf{\Sigma}_{\mathbf{\theta}})$,
$\mathrm{Cov}(\delta_\mathbf{x},\delta_{\mathbf{\theta}})=\mathbf 0$.
Taking the covariance of Eq.~\eqref{eq:lin-perturb} gives the \emph{covariance propagation law}
\begin{equation}\label{eq:cov-propagation}
\mathrm{Cov}(\delta_\mathbf{y})
\approx
J_\mathbf{x}\,\mathbf{\Sigma}_\mathbf{x}\,J_\mathbf{x}^\top
+
J_{\mathbf{\theta}}\,\mathbf{\Sigma}_{\mathbf{\theta}}\,J_{\mathbf{\theta}}^\top.
\end{equation}
\end{theorem}

\begin{proof}
Starting from the linear approximation in Eq.~\eqref{eq:lin-perturb}, write
\begin{equation}
\delta_\mathbf{y} = J_\mathbf{x}\,\delta_\mathbf{x} + J_\mathbf{\theta}\,\delta_{\mathbf{\theta}}.
\end{equation}
Apply the standard formula for the covariance of a linear combination of random vectors:
\begin{equation}
\begin{aligned}
\mathrm{Cov}(A u + B v) &= A\,\mathrm{Cov}(u)A^{\top} + A\,\mathrm{Cov}(u,v)B^{\top} \\
&+ B\,\mathrm{Cov}(v,u)A^{\top} + B\,\mathrm{Cov}(v)B^{\top}.
\end{aligned}
\end{equation}
Here $u = \delta_\mathbf{x}$, $v = \delta_{\mathbf{\theta}}$, $A = J_\mathbf{x}$, $B = J_\mathbf{\theta}$.  Because the perturbations are assumed centered and independent, the cross-covariance terms vanish:
\begin{equation}
\mathrm{Cov}(\delta_\mathbf{x}, \delta_{\mathbf{\theta}})=\mathbf 0.
\end{equation}
Substituting \(\mathrm{Cov}(\delta_\mathbf{x}) = \mathbf{\Sigma}_\mathbf{x}\) and \(\mathrm{Cov}(\delta_{\mathbf{\theta}})=\mathbf{\Sigma}_{\mathbf{\theta}}\) therefore yields
\begin{equation}
\mathrm{Cov}(\delta_\mathbf{y}) = J_\mathbf{x}\,\mathbf{\Sigma}_\mathbf{x}\,J_\mathbf{x}^{\top} + J_\mathbf{\theta}\,\mathbf{\Sigma}_{\mathbf{\theta}}\,J_\mathbf{\theta}^{\top},
\end{equation}
which is the claimed result in Eq.~\eqref{eq:cov-propagation}.  The first term quantifies \emph{how input noise propagates through the network}, whereas the second isolates \emph{parameter-induced uncertainty}.

\end{proof}
\subsection{Per-Parameter Sensitivity Index}
\label{sec:per-param-sensitivity}

Having isolated the parameter-induced component of output variability in Section~\ref{sec:cov-propagation}, we now need to understand how each individual parameter contributes to the second term. To do this, we further decompose the parameter covariance.

Formally, let $\mathbf{\Sigma}_{\mathbf{\theta}}=\operatorname{diag}(\operatorname{Var}(\theta_1),\ldots,\operatorname{Var}(\theta_{d_\theta}))$, and denote the $k$-th column of $J_{\mathbf{\theta}}$ as $\mathbf{j}_k := \partial_{\theta_k}f_{\mathbf{\theta}}(\mathbf{x})\in\mathbb{R}^{d_y}$. Then, we have 
\begin{equation}
  J_{\mathbf{\theta}}\,\mathbf{\Sigma}_{\mathbf{\theta}}\,J_{\mathbf{\theta}}^{\top} 
  = \sum_{k=1}^{d_{\theta}} \operatorname{Var}(\theta_k)\,\mathbf{j}_k\mathbf{j}_k^{\top}.
  \end{equation}
This shows that the total parameter-induced output variance is an additive sum of contributions from each parameter direction. 
Each term $\mathbf{j}_k\mathbf{j}_k^{\top}$ is a rank-1 outer product matrix describing the direction and magnitude by which perturbations of $\theta_k$ affect the output at a specific input $\mathbf{x}$.

A standard scalarization of the output covariance is the total output variance. Therefore, we can obtain it by trace operation:
\vspace{-5mm}
\begin{equation}
  \begin{aligned}
\operatorname{Tr}\big(J_{\mathbf{\theta}}\,\mathbf{\Sigma}_{\mathbf{\theta}}\,J_{\mathbf{\theta}}^{\top} \big) &=\sum_{k=1}^{d_\theta}\operatorname{Var}(\theta_k)\operatorname{Tr}\big(\mathbf{j}_k\mathbf{j}_k^{\top}\big)\\
&=\sum_{k=1}^{d_\theta}\operatorname{Var}(\theta_k)\|\mathbf{j}_k\|_2^2,
\end{aligned}
\label{eq:total-output-variance}
\end{equation}
This equation shows that the local parameter-induced total variance at sample $\mathbf{x}$ is an additive sum of independent per-parameter terms. This naturally identifies the local contribution of $\theta_k$ as 
\begin{equation*}
\delta_{v_k}(\mathbf{x})   := \operatorname{Var}(\theta_k)\|\mathbf{j}_k(\mathbf{x})\|_2^2\\
=\operatorname{Var}(\theta_k)\|\partial_{\theta_k} f_{\mathbf{\theta}}(\mathbf{x})\|_2^2.
\end{equation*}
\emph{Intuition:} $\|\mathbf{j}_k(\mathbf{x})\|_2^2$ measures how strongly the output reacts at $\mathbf{x}$ to an infinitesimal change in $\theta_k$; $\operatorname{Var}(\theta_k)$ determines the typical magnitude of that perturbation. Their product is the \emph{mean-square amplification} of the noise in $\theta_k$ at $\mathbf{x}$.

Averaging the local contribution over the data distribution $\mathcal{D}$ yields the expected contribution of $\theta_k$ to output variance:
\begin{equation}\label{eq:s_k}
  s_k := \operatorname{Var}(\theta_k)
  \,\mathbb E_{\mathbf{x}\sim\mathcal D}
  \Bigl[\|\partial_{\theta_k}f_{\mathbf{\theta}}(\mathbf{x})\|_2^2\Bigr].
\end{equation}
Here, $s_k$ measures the expected contribution of parameter $\theta_k$'s variance to the output variance after amplification through the network. Large $s_k$ implies that small stochastic fluctuations in $\theta_k$ translate into comparatively large output variance, hence $\theta_k$ is \emph{important} (sensitive).
For scalar outputs $(d_y=1)$, $\|\mathbf{j}_k(\mathbf{x})\|_2^2=(\partial_{\theta_k}f_{\mathbf{\theta}}(\mathbf{x}))^2$, so Eq.~\eqref{eq:s_k} reduces to the familiar one-dimensional form. For vector outputs, the Euclidean norm aggregates sensitivities over output coordinates in a basis-invariant way.

\textbf{Practical computation.}  In practice, $s_k$ can be estimated with a single backward pass per mini-batch by accumulating squared gradients weighted by $\operatorname{Var}(\theta_k)$.
\subsection{Connection to Fisher Information}\label{sec:fisher-connection}

The per-parameter sensitivity index derived in Eq.~\eqref{eq:s_k} quantifies how the variance of each parameter contributes to the overall output uncertainty. To better understand its statistical meaning, we now connect this sensitivity measure to the classical concept of Fisher Information, which describes how sensitively the model likelihood responds to infinitesimal parameter perturbations, and thus \emph{serves as a natural bridge between sensitivity analysis and information theory}.
\subsubsection{Approximation of Diagonal Fisher via Gradient Magnitude}

Let $p_{\mathbf{\theta}}(\mathbf{y}|\mathbf{x})$ be the probabilistic output model associated with $f_{\mathbf{\theta}}$.  
The \emph{score function} is
$u_k(\mathbf{x},\mathbf{y};\mathbf{\theta})
\;:=\;\partial_{\theta_k}\log p_{\mathbf{\theta}}(\mathbf{y}|\mathbf{x})$,
whose variance under the data distribution defines the Fisher information matrix (FIM) \cite{kirkpatrick2017overcoming}:
\begin{equation}\label{eq:fim}
\mathcal I_{ij}(\mathbf{\theta})
:=
\mathbb E_{(\mathbf{x},\mathbf{y})\sim\mathcal D}
\bigl[u_i\,u_j\bigr].
\end{equation}

For common choices of $p_{\mathbf{\theta}}$ (softmax for classification, homoscedastic Gaussian for regression), the score factors as
$u_k = T(\mathbf{x},\mathbf{y};\mathbf{\theta})\,\partial_{\theta_k}f_{\mathbf{\theta}}(\mathbf{x})$,
where $T$ is a task-dependent scalar independent of $k$.  Neglecting the variation of $T$ across samples and applying the Generalized Gauss-Newton approximation, one obtains the proportionality
\begin{equation}\label{eq:fisher-diag}
\mathcal I_{kk}\propto\mathbb E_{\mathbf{x}}\bigl[(\partial_{\theta_k}f_{\mathbf{\theta}}(\mathbf{x}))^2\bigr],\;(k=1,\dots,d_{\theta}).
\end{equation}

\begin{theorem}
Comparing Eq.~\eqref{eq:s_k} with Eq.~\eqref{eq:fisher-diag} yields 
\vspace{-1mm}
\begin{equation}\label{eq:s_k-fisher}
  \vspace{-1mm}
s_k = \operatorname{Var}(\theta_k)\,\mathcal I_{kk}.
\end{equation}
which shows that $s_k$ scales the Fisher information by the \emph{prior} parameter uncertainty, converting statistical information content into expected output variance contribution. If $\operatorname{Var}(\theta_k)=1$ (or a constant), then $s_k$ is directly proportional to the diagonal Fisher information.
\end{theorem}

\begin{proof}
We derive the connection between Fisher information and gradient energies in steps. First, for common probabilistic models, the score function can be factorized as
\begin{equation}
  \begin{aligned}
u_k(\mathbf{x},\mathbf{y};\mathbf{\theta})  &= \partial_{\theta_k}\log p_{\mathbf{\theta}}(\mathbf{y}\mid\mathbf{x}) \\
&= T(\mathbf{x},\mathbf{y};\mathbf{\theta})\,\partial_{\theta_k}f_{\mathbf{\theta}}(\mathbf{x}),
\end{aligned}
\end{equation}
Then, we have the diagonal entries of the Fisher information matrix as
\begin{equation*}
\begin{aligned}
\mathcal I_{kk}(\mathbf{\theta}) &= \mathbb{E}_{(\mathbf{x},\mathbf{y})\sim\mathcal{D}}\bigl[u_k^2\bigr]\\
&= \mathbb{E}_{(\mathbf{x},\mathbf{y})\sim\mathcal{D}}\bigl[T^2(\mathbf{x},\mathbf{y};\mathbf{\theta})\,(\partial_{\theta_k}f_{\mathbf{\theta}}(\mathbf{x}))^2\bigr]\\
&= \mathbb{E}_{(\mathbf{x},\mathbf{y})\sim\mathcal{D}}\bigl[T^2(\mathbf{x},\mathbf{y};\mathbf{\theta})\bigr]\,\mathbb{E}_{\mathbf{x}\sim\mathcal{D}}\bigl[(\partial_{\theta_k}f_{\mathbf{\theta}}(\mathbf{x}))^2\bigr],
\end{aligned}
\end{equation*}
where the last step assumes independence between $T^2$ and the gradient term (valid when $T$ depends primarily on the output-target mismatch rather than the specific parameter gradients).

Finally, neglecting the variation of $T^2$ across samples and treating $\mathbb{E}[T^2]$ as a constant $C$ independent of $k$, we obtain
\begin{equation}\label{eq:fisher-diag-approx}
\mathcal I_{kk} \approx C \cdot \mathbb{E}_{\mathbf{x}}\bigl[(\partial_{\theta_k}f_{\mathbf{\theta}}(\mathbf{x}))^2\bigr] \propto \mathbb{E}_{\mathbf{x}}\bigl[(\partial_{\theta_k}f_{\mathbf{\theta}}(\mathbf{x}))^2\bigr].
\end{equation}
Comparing the sensitivity index \eqref{eq:s_k} with the Fisher diagonal \eqref{eq:fisher-diag}, we immediately obtain Eq.~\eqref{eq:s_k-fisher} $s_k = \operatorname{Var}(\theta_k)\,\mathcal I_{kk}$.
\end{proof}

\subsubsection{Interpretation}
Eq.~\eqref{eq:s_k-fisher} shows that the covariance-based sensitivity index $s_k$ converts the {information content} about $\theta_k$ (given by Fisher) into an {expected variance contribution} at the network output by scaling with the a priori uncertainty in $\theta_k$.  Consequently,
\begin{itemize}[nosep]
\item $s_k$ unifies \textbf{statistical} importance (Fisher) with \textbf{Bayesian} uncertainty ($\mathbf{\Sigma}_{\mathbf{\theta}}$).
\item It provides a principled weight for pruning, continual learning regularisers, and domain-invariant parameter selection.
\end{itemize}
\textbf{Impact and Contributions.} The presented framework thus bridges classical information-theoretic analysis with modern large-scale neural modelling in a numerically tractable way.
In addition, it connects the Bayesian uncertainty of parameters with their Fisher information content, thereby deriving a clear sensitivity index for each parameter. This analysis not only provides a novel perspective for understanding model behavior, but also serves as the theoretical foundation for the domain-sensitive parameter regularization method proposed in the following sections.

%% file: sections/method.tex
\section{Domain-Sensitive Parameter Regularization}
\label{sec:method}

\begin{algorithm}[t]
    \caption{Training Algorithm}
    \label{alg:di-training}
    \begin{algorithmic}[1]
    \Require Source domains $\{\mathcal{D}_d\}_{d=1}^D$, regularization strength $\lambda$, update frequency $T_{\mathrm{update}}$
    \Ensure Trained model $f_{\mathbf{\theta}}$ with domain-invariant parameter emphasis
    \State Initialize model parameters $\mathbf{\theta}$
    \State $c \leftarrow \mathbf{1}$ \Comment{Initialize uniform coefficients}
    \For{epoch $= 1, 2, \ldots$}
        \If{epoch $\bmod T_{\mathrm{update}} = 0$}
            \State Compute $\{s_k^{(d)}\}$ for all domains using Eq.~\eqref{eq:domain-sensitivity}
            \State Update $c_k$ using Eq.~\eqref{eq:mean-sensitivity}, \eqref{eq:var-sensitivity}, and \eqref{eq:cv-sensitivity}
        \EndIf
        \For{mini-batch $(\mathbf{x}, \mathbf{y})$ from mixed source domains}
            \State Compute supervised loss $\mathcal{L}_{\mathrm{sup}}$
            \State Compute gradients $\mathbf{g} = \nabla_{\mathbf{\theta}} \mathcal{L}_{\mathrm{sup}}$
            \State Compute regularization $\mathcal{R}_{\mathrm{DS}} = \sum_k c_k \cdot g_k^2$
            \State Update $\mathbf{\theta} \leftarrow \mathbf{\theta} - \eta \nabla_{\mathbf{\theta}}(\mathcal{L}_{\mathrm{sup}} + \lambda \mathcal{R}_{\mathrm{DS}})$
        \EndFor
    \EndFor
    \end{algorithmic}
    \end{algorithm}

Building upon the covariance-based parameter sensitivity framework established in Section~\ref{sec:sensitivity-cov}, we now present a principled approach to identify domain-sensitive parameters for robust domain generalization. Our method leverages the insight that parameters exhibiting consistent sensitivity across source domains are more likely to capture domain-agnostic features essential for generalization.

\subsection{Problem Formulation}

Consider a domain generalization setting with $D$ source domains $\{\mathcal{D}_d\}_{d=1}^D$, where each domain $\mathcal{D}_d$ represents a distinct data distribution over input-output pairs $(\mathbf{x}, \mathbf{y})$. Our goal is to learn a model $f_{\mathbf{\theta}}$ that generalizes well to an unseen target domain $\mathcal{D}_{\mathrm{target}}$ by identifying and preserving domain-invariant parameters while allowing domain-specific adaptation for the remaining parameters.

\subsection{Cross-Domain Sensitivity Analysis}

\textbf{Per-Domain Sensitivity Computation.} For each source domain $\mathcal{D}_d$, we compute the parameter sensitivity index for the $k$-th parameter as:
\begin{equation}\label{eq:domain-sensitivity}
s_k^{(d)} := \operatorname{Var}(\theta_k) \, \mathbb{E}_{\mathbf{x} \sim \mathcal{D}_d} \bigl[(\partial_{\theta_k} f_{\mathbf{\theta}}(\mathbf{x}))^2\bigr],
\end{equation}
where we assume $\operatorname{Var}(\theta_k) = 1$ for simplicity (or use empirical parameter variance if available).

In practice, this is efficiently computed as:
\begin{equation}
s_k^{(d)} \approx \frac{1}{|\mathcal{B}_d|} \sum_{\mathbf{x} \in \mathcal{B}_d} (\partial_{\theta_k} \mathcal{L}(\mathbf{x}, \mathbf{y}))^2,
\end{equation}
where $\mathcal{B}_d$ represents mini-batches sampled from domain $\mathcal{D}_d$, and $\mathcal{L}$ is the supervised loss function.

\textbf{Domain-Invariance Quantification.} To measure how consistently a parameter behaves across domains, we compute cross-domain statistics:
\begin{align}
\bar{s}_k &:= \frac{1}{D} \sum_{d=1}^D s_k^{(d)}, \label{eq:mean-sensitivity}\\
v_k &:= \frac{1}{D} \sum_{d=1}^D (s_k^{(d)} - \bar{s}_k)^2, \label{eq:var-sensitivity}\\
c_k &:= \frac{\sqrt{v_k}}{\bar{s}_k + \epsilon}, \label{eq:cv-sensitivity}
\end{align}
where $\bar{s}_k$ is the mean sensitivity, $v_k$ is the variance across domains, and $c_k$ is the coefficient of variation. The parameter $\epsilon > 0$ is a small constant to avoid division by zero.

\textbf{Interpretation:} A small coefficient of variation $c_k$ indicates that parameter $\theta_k$ exhibits consistent sensitivity across all source domains, suggesting it captures domain-invariant features. Conversely, large $c_k$ values indicate domain-specific sensitivity patterns.

\subsection{Soft Regularization for Domain Generalization}

Rather than hard thresholding to select domain-invariant parameters, we employ a soft regularization approach that continuously penalizes domain-sensitive parameters during training.

\textbf{Regularized Objective.} We augment the standard empirical risk minimization objective with a domain-sensitivity-aware regularization term:
\begin{equation}\label{eq:regularized-loss}
\mathcal{L}_{\mathrm{total}} = \mathcal{L}_{\mathrm{sup}} + \lambda \mathcal{R}_{\mathrm{DS}}(\mathbf{\theta}),
\end{equation}
where $\mathcal{L}_{\mathrm{sup}}$ is the supervised loss across all source domains, and $\mathcal{R}_{\mathrm{DS}}(\mathbf{\theta})$ is our domain-sensitive regularizer:
\begin{equation}\label{eq:di-regularizer}
\mathcal{R}_{\mathrm{DS}}(\mathbf{\theta}) := \sum_{k=1}^{d_\theta} c_k \cdot (\partial_{\theta_k} \mathcal{L}_{\mathrm{sup}})^2.
\end{equation}

\textbf{Discussion:} This regularizer applies stronger penalties to parameters with high cross-domain sensitivity variation ($c_k$), effectively encouraging the optimization to rely more heavily on domain-invariant parameters while suppressing domain-specific adaptations. In addition, the sensitivity coefficients $\{c_k\}$ are updated periodically (every $T_{\text{update}}$ iterations) rather than at each iteration to balance computational cost with adaptation to changing parameter importance. The gradient computation for $\mathcal{R}_{\text{DS}}$ adds minimal overhead as it reuses gradients already computed for the supervised loss.

\begin{table*}[t]
	\begin{center}
	
		\caption{\textbf{Comparison with state-of-the-art domain generalization methods.} Out-of-domain accuracies on five domain generalization benchmarks are shown. Top performing methods are highlighted in \textbf{bold} while second-best are \textit{underlined}. Our experiments are repeated three times.}

	\vspace{-1em}
	\begin{tabular}{p{3cm}|ccccc|c}
	    \toprule
		Algorithm & PACS          & VLCS          & OfficeHome    & {TerraInc}    &DomainNet  &  {Avg.} \\
		\midrule
		MMD \cite{li2018domain}                  & 
		84.7\scriptsize{$\pm0.5$}          
		 & 
		77.5\scriptsize{$\pm0.9$}            
		 & 
		66.3\scriptsize{$\pm0.1$}          
		 & 
		42.2\scriptsize{$\pm1.6$}            
		   &      
		23.4\scriptsize{$\pm9.5$}           &  58.8 \\
        MLDG \cite{li2018learning}                & 
		84.9\scriptsize{$\pm1.0$}          
		 &
		77.2\scriptsize{$\pm0.4$}          
		 & 
		66.8\scriptsize{$\pm0.6$}          
		 &
		47.7\scriptsize{$\pm0.2$}          
		   &
		41.2\scriptsize{$\pm0.1$}            & 63.6  \\

        IRM \cite{arjovsky2019invariant}            & 
		83.5\scriptsize{$\pm0.8$}          
		& 
		78.5\scriptsize{$\pm0.5$}          
		 & 
		64.3\scriptsize{$\pm2.2$}          
		& 
		47.6\scriptsize{$\pm0.8$}          
		       &
		33.9\scriptsize{$\pm2.8$}          & 61.6 \\
		Mixstyle \cite{zhou2021mixstyle}     & 
		85.2\scriptsize{$\pm0.3$} 
		         & 
		77.9\scriptsize{$\pm0.5$}           & 60.4\scriptsize{$\pm0.3$}           & 44.0\scriptsize{$\pm0.7$}
		                   &
		34.0\scriptsize{$\pm0.1$}           & 60.3 \\
		
		GroupDRO \cite{sagawadistributionally}    & 
		84.4\scriptsize{$\pm0.8$}          
		   & 
		76.7\scriptsize{$\pm0.6$}          
		   & 
		66.0\scriptsize{$\pm0.7$}             
		   & 
		43.2\scriptsize{$\pm1.1$}              
		           &
		33.3\scriptsize{$\pm0.2$}          & 60.7 \\
		
		ARM \cite{zhang2020adaptive}               &
		85.1\scriptsize{$\pm0.4$}          
		& 
		77.6\scriptsize{$\pm0.3$} 
		           & 
		64.8\scriptsize{$\pm0.3$} 
		           & 
		45.5\scriptsize{$\pm0.3$} 
		                    &
		35.5\scriptsize{$\pm0.2$}          & 61.7 \\
		
		MTL \cite{blanchard2021domain}    & 
		84.6\scriptsize{$\pm0.5$}           
		 & 
		77.2\scriptsize{$\pm0.4$}          
		 & 
		66.4\scriptsize{$\pm0.5$}          
		 & 
		45.6\scriptsize{$\pm1.2$}          
		        &
		40.6\scriptsize{$\pm0.1$}          & 62.9 \\

		VREx \cite{krueger2021out}           & 
		84.9\scriptsize{$\pm0.6$} 
		           & 
		78.3\scriptsize{$\pm0.2$} 
		           & 
		66.4\scriptsize{$\pm0.6$} 
		           & 
		46.4\scriptsize{$\pm0.6$} 
		              &
		33.6\scriptsize{$\pm2.9$}          & 61.9   \\

		Mixup\cite{xu2020adversarial}             & 
		84.6\scriptsize{$\pm0.6$}            
		 & 
		77.4\scriptsize{$\pm0.6$}          
		& 
		68.1\scriptsize{$\pm0.3$}            
		& 
		47.9\scriptsize{$\pm0.8$}            
		     &
		39.2\scriptsize{$\pm0.1$}            & 63.4    \\

		SagNet \cite{nam2021reducing}           &             
		86.3\scriptsize{$\pm0.2$} 
		 & 
		77.8\scriptsize{$\pm0.5$}          
		& 
		68.1\scriptsize{$\pm0.1$}          
		& 
		48.6\scriptsize{$\pm1.0$}          
		        &
		40.3\scriptsize{$\pm0.1$}          &  64.2 \\

		CORAL \cite{sun2016deep}             & 
		86.2\scriptsize{$\pm0.3$}          
		& 78.8\scriptsize{$\pm0.6$}          
		 & 
		68.7\scriptsize{$\pm0.3$}          
		 & 
		47.6\scriptsize{$\pm1.0$}          
		     &
		41.5\scriptsize{$\pm0.1$}          & 64.5   \\
		
		
		RSC \cite{huang2020self}               & 
		85.2\scriptsize{$\pm0.9$}          
		& 77.1\scriptsize{$\pm0.5$}          
		& 
		65.5\scriptsize{$\pm0.9$}         & 
		46.6\scriptsize{$\pm1.0$}          
		   &
		38.9\scriptsize{$\pm0.5$}          & 62.7 \\
        
           ERM \cite{shahtalebi2021sand} & 
		85.5\scriptsize{$\pm0.9$}  
		& 
		77.5\scriptsize{$\pm0.9$}  
		& 
		66.5\scriptsize{$\pm0.4$}  
		 & 
		46.1\scriptsize{$\pm0.3$} 
		
		     &
		40.9\scriptsize{$\pm0.6$} & 63.3 \\ 
		
		SAM \cite{foretsharpness}  & 
		85.8\scriptsize$\pm0.2$             
		 & 
		79.4\scriptsize$\pm0.1$              
		& 
		69.6\scriptsize$\pm0.1$              
		 & 
		43.3\scriptsize$\pm0.7$              
		   &
		44.3\scriptsize$\pm0.0$              &  64.5 \\

        Fish \cite{shigradient}                     & 
		85.5\scriptsize{$\pm0.3$}          
		 & 
		77.8\scriptsize{$\pm0.3$}          
		 & 
		68.6\scriptsize{$\pm0.4$}            
		 & 
		45.1\scriptsize{$\pm1.3$}            
		    &
		42.7\scriptsize{$\pm0.2$}            &   63.9    \\
		GSAM \cite{zhuangsurrogate}  & 
		85.9\scriptsize$\pm0.1$           
		 & 
		79.1\scriptsize$\pm0.2$             
		 & 
		69.3\scriptsize$\pm0.0$             
		& 
		47.0\scriptsize$\pm0.8$             
		   &
		44.6\scriptsize$\pm0.2$             &   65.1 \\
		
		SAGM \cite{wang2023sharpness}       & 
		\underline{86.6}\scriptsize{$\pm0.2$}           
		& 
		\underline{80.0}\scriptsize{$\pm0.3$}           
		& 
		\underline{70.1}\scriptsize{$\pm0.2$}           
		& 
		48.8\scriptsize{$\pm0.9$}           
		  &
		45.0\scriptsize{$\pm0.2$}           &  66.1 \\

        GMDG \cite{tan2024rethinking}       & 
		\underline{85.6}\scriptsize{$\pm0.3$}           
		& 
		79.2\scriptsize{$\pm0.3$}           
		& 
		\textbf{70.7}\scriptsize{$\pm0.2$}           
		& 
		\textbf{51.1}\scriptsize{$\pm0.9$}           
		  &
		\underline{44.6}\scriptsize{$\pm0.2$}           &  66.3 \\
        
        GGA \cite{ballas2025gradient}       & 
		\underline{87.3}\scriptsize{$\pm0.2$}           
		& 
		79.9\scriptsize{$\pm0.3$}           
		& 
		68.5\scriptsize{$\pm0.2$}           
		& 
		50.6\scriptsize{$\pm0.9$}           
		  &
		\underline{45.2}\scriptsize{$\pm0.2$}           &  66.3 \\
		
		\midrule
		\textbf{DSP-Reg}                & 
		\textbf{87.5}\scriptsize{$\pm0.3$}           & 
		\textbf{80.1}\scriptsize{$\pm0.4$}           & 
		69.4\scriptsize{$\pm0.3$}           &  
		\underline{50.7}\scriptsize{$\pm0.2$}    &
		\textbf{45.6}\scriptsize{$\pm0.2$} & \textbf{66.7}\\

		\bottomrule
		
	\end{tabular}
    \label{table:total-results}
    \end{center}
	\vspace{-2em}
\end{table*}

\subsection{Theoretical Justification}

\textbf{Generalization Bound.} Our approach can be theoretically motivated through the lens of domain adaptation theory \cite{10.1007/s10994-009-5152-4}. Let $\mathcal{H}$ denote the hypothesis class and $d_{\mathcal{H}}(\mathcal{D}_s, \mathcal{D}_t)$ be the $\mathcal{H}$-divergence between source and target domains, and $\epsilon_s(\mathbf{\theta})$ be the source domain error. The target domain error $\epsilon_t(\mathbf{\theta})$ can be bounded as:
\begin{equation}
\epsilon_t(\mathbf{\theta}) \leq \epsilon_s(\mathbf{\theta}) + d_{\mathcal{H}}(\mathcal{D}_s, \mathcal{D}_t) + \lambda^*,
\end{equation}
where $\lambda^*$ is the optimal joint error. By penalizing parameters whose sensitivities vary across source domains, our method implicitly constrains the hypothesis space to functions that behave consistently across domains, thereby reducing the effective $\mathcal{H}$-divergence term and tightening the generalization bound.

\textbf{Connection to Invariant Risk Minimization.} Our regularizer shares conceptual similarities with Invariant Risk Minimization (IRM) but operates at the parameter level rather than the representation level. While IRM seeks representations that are optimal across all domains, our method identifies which parameters contribute to such representations, providing finer-grained control over the invariance-adaptation trade-off.

\textbf{Connection to Fisher Information.} From the relationship $s_k = \operatorname{Var}(\theta_k) \mathcal{I}_{kk}$ established in Section~\ref{sec:sensitivity-cov}, our regularizer can be interpreted as a domain-consistency-weighted Fisher information penalty. Parameters with consistent Fisher information across domains receive lower penalties, aligning with the principle of preserving statistically important yet domain-agnostic features.

%% file: sections/experiments.tex
\section{Experiments}
\subsection{Datasets}
To evaluate the effectiveness of our proposed Domain-Sensitive Parameter Regularization (DSP-Reg) framework, we conduct extensive experiments on five widely-adopted domain generalization benchmarks, which include PACS \cite{li2017deeper}, a standard benchmark consisting of 9,991 images across 7 classes from 4 distinct domains (Photo, Art Painting, Cartoon, and Sketch); VLCS \cite{fang2013unbiased}, which comprises 10,729 images across 5 classes from 4 different datasets (PASCAL VOC2007, LabelMe, Caltech-101, and SUN09); the challenging OfficeHome dataset \cite{venkateswara2017deep} with 15,588 images spanning 65 classes and 4 domains (Artistic, Clipart, Product, and Real-World); TerraIncognita \cite{beery2018recognition}, a real-world dataset containing 24,788 images of wild animals across 10 classes from 4 different locations; and finally, the large-scale DomainNet benchmark \cite{peng2019moment}, with approximately 586,000 images across 345 classes and 6 domains. Together, these datasets cover a diverse range of domain shifts, including style, context, and viewpoint variations, providing a comprehensive testbed for our method.

\subsection{Evaluation Protocol}
For fair and robust comparison, our entire experimental procedure adheres strictly to the protocol established by DomainBed \cite{gulrajani2021in}. We employ the standard \emph{leave-one-domain-out} (LODO) cross-validation strategy for all datasets. In each run, one domain is held out as the unseen target domain for testing, while the remaining domains are used for training. The model selection and hyperparameter tuning are performed on a validation set created from the training domains. The final performance is reported as the average top-1 accuracy on the target domain, and all experiments are repeated over three random seeds to ensure statistical significance. 

\begin{table}[t]
    \caption{\textbf{Ablation study on PACS dataset.} We systematically evaluate each component of DSP-Reg. Results show leave-one-domain-out accuracy for each target domain.}
    \vspace{-3mm}
    \centering
    \begin{tabular}{l|cccc|c}
    \toprule
    Method & Photo & Art & Cartoon & Sketch & Average \\
    \midrule
    Full Method & \textbf{98.9} & \textbf{86.9} & \textbf{82.8} & \textbf{81.4} & \textbf{87.5} \\
    w/o $\mathcal{R}_{\mathrm{DI}}$ ($\lambda = 0$) & 97.2 & 84.7 & 80.8 & 79.3 & 85.5 \\
    w/o $c_k$ ($c_k = 1$) & 97.5 & 84.9 & 81.2 & 79.6 & 85.8 \\
    \bottomrule
    \end{tabular}
    \label{table:ablation-main}
    \vspace{-5mm}
    \end{table}

\subsection{Implementation Details.} Our implementation follows standard practices for all experiments. We use the ResNet-50 \cite{he2016deep} pre-trained on ImageNet \cite{russakovsky2015imagenet}, with its batch normalization layers frozen during fine-tuning, and the Adam optimizer for model training. General hyperparameters such as learning rate, weight decay, and dropout rates are selected via random search from the distributions specified in the DomainBed framework. We set the batch size to 32. Our proposed method introduces two primary hyperparameters: the regularization strength $\lambda$ and the sensitivity coefficient update frequency $T_{\mathrm{update}}$, which are tuned for each LODO split using the validation set. Specifically, $\lambda$ is 0.001 and $T_{\mathrm{update}}$ is set to 2. We conduct all experiments on four NVIDIA A100 GPUs.

\subsection{Performance Comparison}
To validate the effectiveness of our proposed DSP-Reg, we
report the average OOD performances on a total of five DG benchmarks in Table \ref{table:total-results}, The results for each domain are reported in the Appendix. Our method achieves superior performance across all datasets, validating its effectiveness in capturing and utilizing domain-invariant parameter sensitivities. Specifically, DSP-Reg achieves an impressive average accuracy of 66.7\%, surpassing state-of-the-art methods such as GMDG (66.3\%), GGA (66.3\%), and SAGM (66.1\%). On individual benchmarks, our method notably achieves the highest accuracy on PACS (87.5\%), VLCS (80.1\%) and DomainNet (45.6\%). 
Additionally, our method demonstrates competitive performance on TerraIncognita (50.7\%), slightly surpassing previous approaches. On OfficeHome, DSP-Reg attains an accuracy of 69.4\%, closely trailing behind GMDG yet significantly outperforming most other approaches. 
These results collectively illustrate the robustness and versatility of DSP-Reg in handling diverse domain shifts. Compared with the strong Empirical Risk Minimization (ERM) baseline, DSP-Reg significantly improves average performance by 3.4\%, highlighting the benefits of parameter-level regularization guided by covariance-based sensitivity analysis.

\subsection{Ablation Studies and Component Analysis}

To evaluate the effectiveness of each component in our proposed Domain-Sensitive Parameter Regularization (DSP-Reg) framework, we conduct comprehensive ablation studies on the PACS dataset. These experiments systematically examine the contribution of each key component while maintaining all other experimental settings constant.

\begin{figure}[t]
    \centering
	
    \includegraphics[width=.6\linewidth]{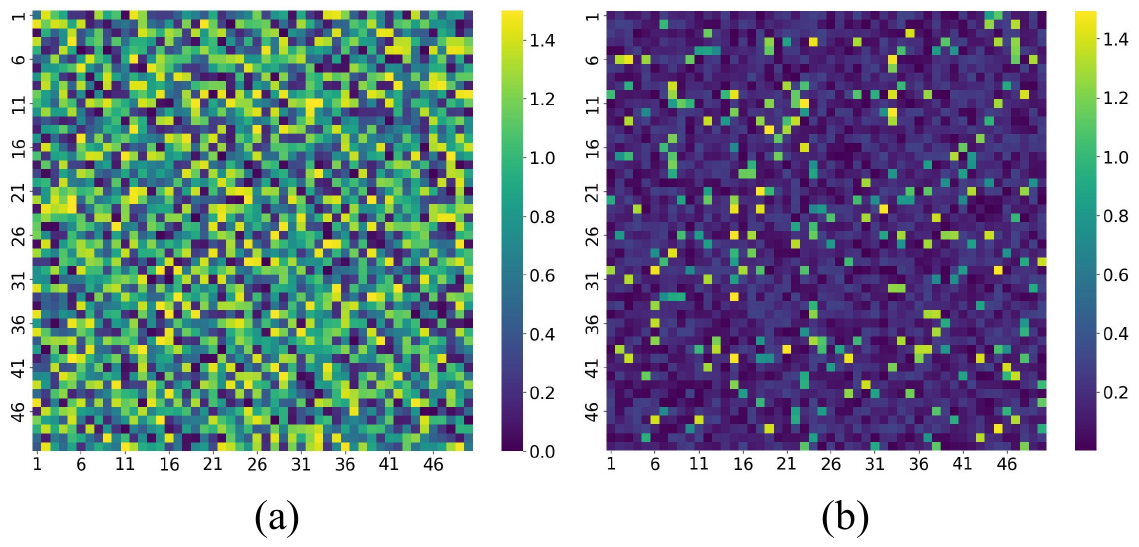}
	\vspace{-3mm}
    \caption{Visualization of parameter sensitivities $c_k$ for 2500 randomly sampled parameters from the 13th layer of ResNet-50. (a) Initial $c_k$ values computed at the beginning of training; (b) $c_k$ values after completion of training. Brighter colors indicate higher domain sensitivity.}
    \label{fig:map}
    \end{figure}

\textbf{Effect of Domain-Invariance Regularization.} We first investigate the core contribution of our domain-invariance regularization term $\mathcal{R}_{\mathrm{DI}}(\mathbf{\theta})$ by comparing our full method against a baseline without regularization. This is achieved by setting the regularization strength $\lambda = 0$ in Eq.~\eqref{eq:regularized-loss}, effectively reducing our method to standard ERM. As shown in Table~\ref{table:ablation-main}, the addition of $\mathcal{R}_{\mathrm{DI}}(\mathbf{\theta})$ provides substantial improvements across all four domains of PACS, with an average accuracy gain of 2\%. This demonstrates the fundamental importance of encouraging domain-invariant parameter behavior through our sensitivity-based regularization. Figure \ref{fig:map} provides an intuitive visualization of how domain sensitivities $c_k$ evolve during training for parameters sampled from the 13th layer of ResNet-50. Initially, as shown in Figure \ref{fig:map}(a), parameters exhibit relatively high and diverse $c_k$ values, reflecting significant variation in sensitivity across different domains. After training with our proposed domain-invariant regularization (Figure \ref{fig:map}(b)), most parameters'$c_k$ values substantially decrease, indicating enhanced domain invariance. This demonstrates that our method effectively identifies and suppresses domain-sensitive parameters, thereby contributing to improved robustness and generalization across unseen domains.

\textbf{Necessity of Cross-Domain Sensitivity Coefficients.} To evaluate the importance of our cross-domain consistency weighting mechanism, we compare our method against a variant where all sensitivity coefficients are set to unity ($c_k = 1$ for all parameters $k$). This ablation, denoted as ``w/o $c_k$" in Table~\ref{table:ablation-main}, reduces our regularizer from Eq.~\eqref{eq:di-regularizer} to a standard gradient penalty. The results reveal that uniform weighting significantly underperforms our adaptive weighting scheme, with a 1.7\% drop in average accuracy. This confirms that the coefficient of variation $c_k$ from Eq.~\eqref{eq:cv-sensitivity} effectively identifies which parameters require stronger regularization based on their cross-domain sensitivity patterns.

\textbf{Sensitivity Coefficient Update Interval.} Our method periodically updates the sensitivity coefficients $\{c_k\}$ every $T_{\mathrm{update}}$ epoch as described in Algorithm~\ref{alg:di-training}. We investigate the impact of different update frequencies by varying $T_{\mathrm{update}} \in \{1, 2, 3, 4\}$ epochs. Figure~\ref{fig:accuracy} illustrates that updating every 2 epochs achieves the optimal balance between computational efficiency and adaptation to changing parameter importance. More frequent updates ($T_{\mathrm{update}} = 1$) introduce unnecessary computational overhead without significant performance gains, while less frequent updates ($T_{\mathrm{update}} \geq 3$) fail to capture the evolving sensitivity patterns during training, leading to suboptimal regularization.

\textbf{Dynamic vs. Static Sensitivity Coefficients.} To assess the value of dynamically updating sensitivity coefficients throughout training, we compare our full method against a variant that computes $c_k$ only once at the beginning and maintains these values for the remainder of training. This "Static $c_k$" variant, shown in Table~\ref{table:ablation-dynamic}, demonstrates that dynamic updating provides a consistent 1.8\% improvement in average accuracy. This result underscores the importance of adapting the regularization weights as the model learns, since parameter sensitivity patterns evolve during the optimization process.

\begin{table}[t]
    \caption{\textbf{Dynamic vs. static sensitivity coefficient updates.} Comparison between our dynamic updating scheme and computing coefficients only once during training.}
    \vspace{-3mm}
    \centering
\begin{tabular}{l|cccc|c}
\toprule
Method & Photo & Art & Cartoon & Sketch & Average \\
\midrule
Dynamic $c_k$ & \textbf{98.9} & \textbf{86.9} & \textbf{82.8} & \textbf{81.4} & \textbf{87.5} \\
Static $c_k$ & 97.5 & 84.7 & 81.0 & 79.6 & 85.7 \\
\bottomrule
\end{tabular}
\label{table:ablation-dynamic}
\end{table}

\textbf{Impact of $\lambda$ on Performance.} We conduct a sensitivity analysis of the regularization strength $\lambda$ in Eq.~\eqref{eq:regularized-loss} by evaluating our method with $\lambda \in \{0.0001, 0.001, 0.01, 0.1\}$. Figure~\ref{fig:accuracy} shows that performance increases with $\lambda$ up to the optimal value of $\lambda = 0.001$, beyond which excessive regularization begins to harm the supervised learning objective. This demonstrates that our method benefits from moderate regularization strength that balances domain-invariance with task-specific learning.


\begin{figure}[t]
    \centering
    \begin{subfigure}[b]{0.3\linewidth}
    \includegraphics[width=\textwidth]{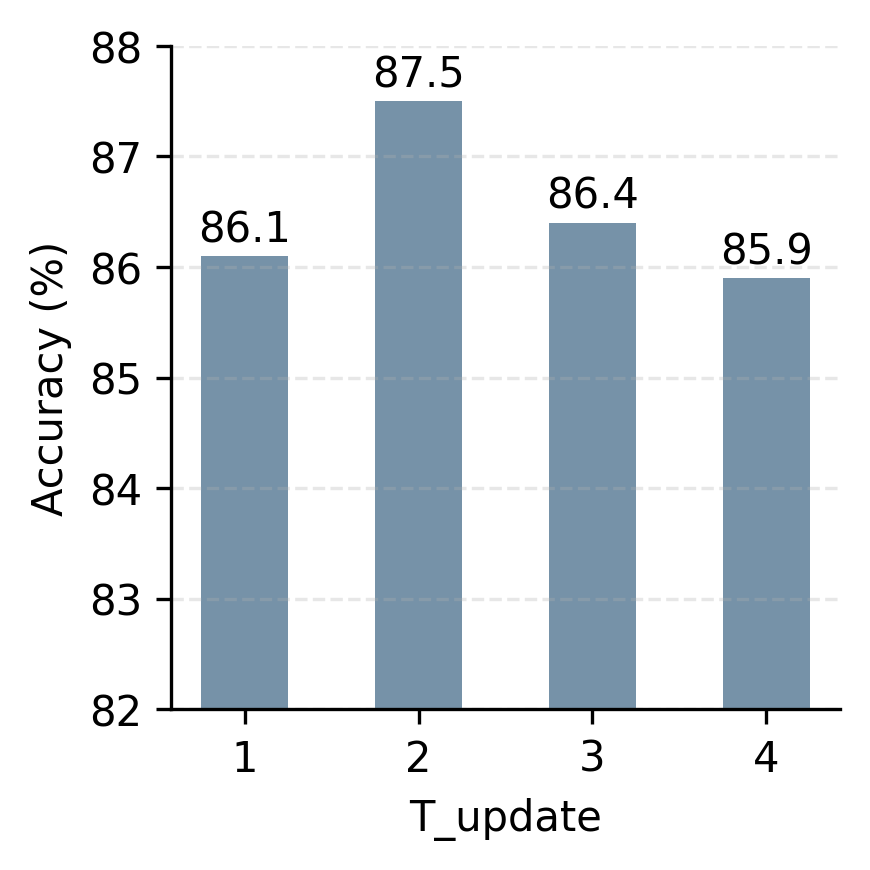}
    \caption{}
    \end{subfigure}
    \begin{subfigure}[b]{0.3\linewidth}
    \includegraphics[width=\textwidth]{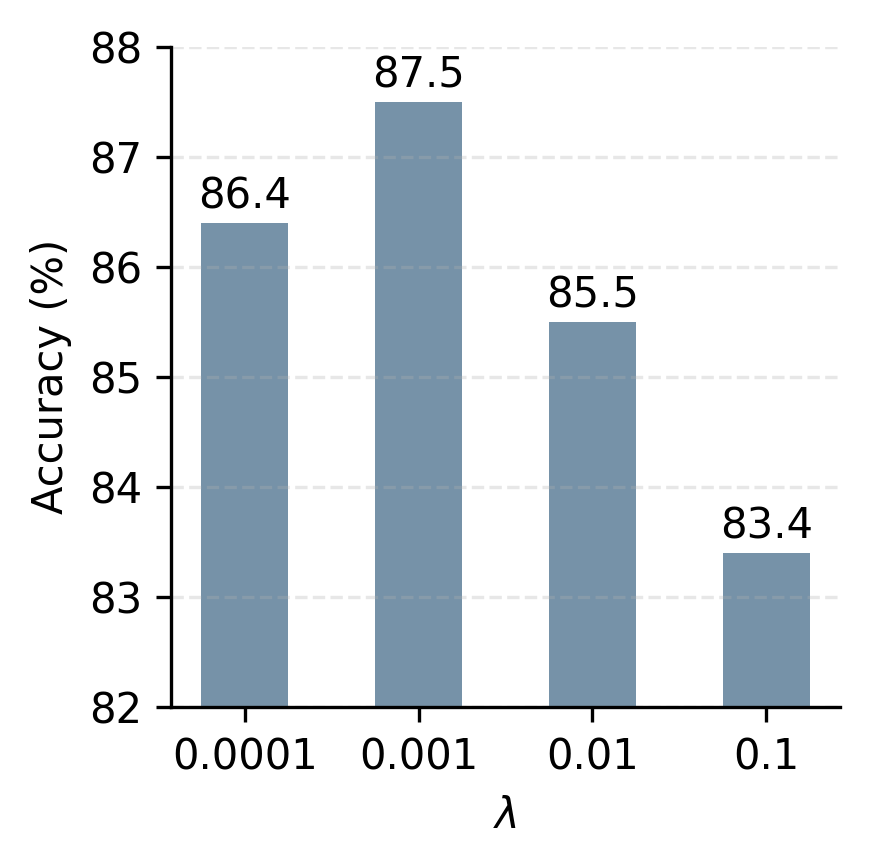}
    \caption{}
    \end{subfigure}
    \vspace{-3mm}
    \caption{\textbf{(a) Impact of the update frequency of sensitivity coefficient.} Average accuracy on PACS dataset for different values of $T_{\mathrm{update}}$. Updating every 2 epochs provides the optimal trade-off between performance and computational cost. \textbf{(b) Regularization strength analysis.} Average accuracy on PACS dataset as a function of the regularization parameter $\lambda$. The optimal value $\lambda = 0.001$ balances domain-invariance with supervised learning.}
    \label{fig:accuracy}
\end{figure}

\textbf{Computational Overhead Analysis}
To evaluate the practical scalability of DSP-Reg, we performed a computational overhead analysis focusing on training speed (measured in iterations per second) and peak GPU memory usage as shown in Table \ref{table:computational-overhead}. All experiments are under the same training configuration as the main experiments. Compared with the standard ERM baseline, DSP-Reg introduces only minor additional cost. The slight reduction in iterations per second primarily results from the cross-domain sensitivity statistics computed during periodic updates, while the computation of the regularisation term itself is negligible since it reuses the gradients from the supervised loss. DSP-Reg exhibits an approximately 3.4\% decrease in training speed and around 1.4 GB increase in peak GPU memory relative to ERM. These results confirm that DSP-Reg maintains practical scalability with minimal computational overhead.

\begin{table}[t]
    \caption{\textbf{Computational overhead analysis.}}
    \vspace{-3mm}
    \centering
    \begin{tabular}{l|cc}
    \toprule
    Method & Iterations / second & Peak GPU memory (GB) \\
    \midrule
    ERM (baseline) & \textbf{2.9} & 22.9  \\
    DSP-Reg ($T_{\mathrm{update}}$ = 2) & 2.8 & \textbf{24.3}  \\
    \bottomrule
    \end{tabular}
    \label{table:computational-overhead}
    \end{table}

%% file: sections/conclusion_ackno.tex
\section{Conclusion}
In this paper, we have introduced Domain-Sensitive Parameter  Regularization (DSP-Reg) framework, a principled approach aimed at enhancing domain generalization by explicitly quantifying and regularizing parameter-level domain sensitivity. By employing a covariance-based sensitivity analysis framework, we define a robust per-parameter metric, the cross-domain coefficient of variation, to effectively identify and penalize domain-sensitive parameters. Extensive experiments across five challenging DG benchmarks demonstrate the superiority of DSP-Reg and achieve state-of-the-art results. The consistent improvements underscore the importance of fine-grained, parameter-level regularization in achieving robust model generalization.




%% file: sections/Appendix.tex
\appendix

\makeatletter
\def\cvprsect#1{\cvprsection{\texorpdfstring{\hskip -0.8em.~}{}#1}}
\makeatother


\section{Discussion of Positioning and Uniqueness of DSP-Reg}

Existing parameter-level domain generalization methods, such as Fishr \citep{rame2022fishr}, FisherTune \citep{zhao2025fishertune}, or gradient-alignment regularizers \cite{huangSelfChallengingImprovesCrossDomain2020, 10203522,10376731,ballas2025gradient}, rely on statistical consistency of gradients or Fisher information across domains. These approaches typically enforce equality of global statistics, for example, matching gradient variances or Fisher diagonals, which leads to treating all parameters uniformly. In contrast, DSP-Reg introduces a fine-grained, covariance-based view that explicitly decomposes parameter uncertainty into per-parameter contributions and quantifies their cross-domain variability.

Specifically, DSP-Reg defines a per-parameter sensitivity index $s_k$ derived from second-order moment propagation, then measures the coefficient of variation $c_k$ of $s_k$ across domains. This formulation provides a direct, interpretable estimate of how each parameter’s Fisher curvature fluctuates under domain shift. By penalizing parameters with large $c_k$, DSP-Reg achieves domain-consistency weighting at the parameter level, something that global Fisher regularizers or representation-level invariance methods cannot provide.

Furthermore, DSP-Reg is \emph{orthogonal and complementary to existing DG paradigms}: it can be seamlessly integrated with data-augmentation, invariant-representation, or meta-learning schemes as an additional regularization term, introducing minimal computational overhead while offering a new parameter-sensitivity perspective on domain generalization.

\section{Theoretical Rationale: Cross-Domain Sensitivity Disparity Indicates Domain-Specificity}

This section provides a theoretical justification for the claim that \emph{if a parameter exhibits significant sensitivity differences across source domains, it is likely encoding domain-specific rather than domain-invariant regularities}.
We proceed from first principles without invoking additional modeling assumptions except for local differentiability and finite second moments.
Specifically, we 1) linearise the neural network to relate infinitesimal input or parameter perturbations to output changes, 2) propagate second-order moments to isolate the \emph{parameter-induced} component of output variability, and 3) define a per-parameter, per-domain sensitivity $s_k^{(d)}$ and its cross-domain dispersion $c_k$.
We also discuss the underlying assumptions (locality, centred moments, weak cross-terms) and provide simple empirical checks.

\vspace{.5em}
\subsection{Notation and Goal}
Let $f_\theta:\mathbb{R}^{d_x}\!\to\!\mathbb{R}^{d_y}$ be differentiable.
Domains are $\{D_d\}_{d=1}^D$.
For a parameter index $k$, denote the Jacobian column $\mathbf{j}_k(x)\!:=\!\partial_{\theta_k} f_\theta(x)\in\mathbb{R}^{d_y}$.
For domain $D_d$, we can define the sensitivity index $s_k^{(d)}$ as:
\begin{equation}
s_k^{(d)} := \Var(\theta_k)\;\E_{x\sim D_d}\!\left[\|\mathbf{j}_k(x)\|_2^2\right],
\qquad
\bar s_k := \frac{1}{D} \sum_{d=1}^D s_k^{(d)},
\qquad
v_k := \frac{1}{D} \sum_{d=1}^D \big(s_k^{(d)}-\bar s_k\big)^2,
\end{equation}
and the coefficient of variation $c_k := \sqrt{v_k}/(\bar s_k+\varepsilon)$ with tiny $\varepsilon>0$.

\paragraph{Claim:} If $c_k$ is large (equivalently, $\{s_k^{(d)}\}$ differ significantly across different domains $D_d$), then parameter $\theta_k$ is more likely encoding domain-specific rather than domain-invariant features.

\paragraph{Idea.}
$s_k^{(d)}$ measures how strongly infinitesimal changes in $\theta_k$ impact the output on domain $D_d$.
If $\theta_k$ relied only on domain-invariant statistics, that shaking should be (approximately) the same for all domains $D_d$.
Therefore, large cross-domain dispersion of $s_k^{(d)}$ contradicts domain-invariance.

\vspace{1em}

\subsection{From Perturbations to Sensitivity}

\paragraph{Step 1: Local linearisation (by differentiability).}
At $(x_0,\theta_0)$, for small perturbations $(\delta_x,\delta_\theta)$,
\begin{equation}
\delta_y \;\approx\; J_x\,\delta_x \;+\; J_\theta\,\delta_\theta,
\qquad
J_x \!=\! \partial_x f_\theta(x_0),\ \ 
J_\theta \!=\! \partial_\theta f_\theta(x_0).
\label{eq:lin}
\end{equation}
\textit{Intuition.} This means that how much the output moves if we wiggle input or parameters just a little.

\paragraph{Step 2: Second-order moment propagation (by centeredness).}
Assume centered perturbations with finite second moments:
$\E[\delta_x]=0$, $\E[\delta_\theta]=0$.
Let $\Sigma_x\!=\!\Cov(\delta_x)$, $\Sigma_\theta\!=\!\Cov(\delta_\theta)$, and initially ignore cross-covariance (we relax later).
Taking covariance on both sides of Eq. \eqref{eq:lin} gives
\begin{align}
\Cov(\delta_y) 
&\approx \Cov(J_x\delta_x) + \Cov(J_\theta\delta_\theta) 
\quad (\text{drop }\Cov(J_x\delta_x, J_\theta\delta_\theta) \text{ for now}) \nonumber\\
&= J_x \Cov(\delta_x) J_x^\top + J_\theta \Cov(\delta_\theta) J_\theta^\top \nonumber\\
&= J_x \Sigma_x J_x^\top \;+\; J_\theta \Sigma_\theta J_\theta^\top .
\label{eq:cov-split}
\end{align}
\textit{Intuition.} Output variance splits into an input-induced term and a parameter-induced term.

\paragraph{Step 3: Isolating the $k$-th parameter contribution (by diagonality).}
Parameter perturbances along different coordinates are often modeled independent locally, so let
\begin{equation}
\Sigma_\theta=\diag\!\big(\Var(\theta_1),\ldots,\Var(\theta_{d_\theta})\big), 
\qquad J_\theta = [\, \mathbf{j}_1(x_0)\ \cdots\ \mathbf{j}_{d_\theta}(x_0) \,].
\end{equation}
Then
\begin{align}
J_\theta \Sigma_\theta J_\theta^\top 
&= \sum_{k=1}^{d_\theta} \Var(\theta_k)\, \mathbf{j}_k(x_0)\, \mathbf{j}_k(x_0)^\top.
\label{eq:param-cov-sum}
\end{align}
Taking expectation over $x\sim D_d$ and allowing $x_0$ to vary with samples, the \emph{expected parameter-induced variance} aggregates as
\begin{equation}
\underbrace{\E_{x\sim D_d}\!\big[J_\theta(x)\Sigma_\theta J_\theta(x)^\top\big]}_{\text{parameter-induced output variance on }D_d}
= \sum_{k=1}^{d_\theta} \Var(\theta_k)\ \E_{x\sim D_d}\!\big[\mathbf{j}_k(x)\mathbf{j}_k(x)^\top\big].
\label{eq:domain-expect}
\end{equation}
Projecting onto the $k$-th coordinate (i.e., attributing the variance share to $\theta_k$) gives the scalar
\begin{equation}
s_k^{(d)} 
\;\triangleq\; \Var(\theta_k)\; \E_{x\sim D_d}\!\big[\|\mathbf{j}_k(x)\|_2^2\big].
\label{eq:s-def}
\end{equation}

\paragraph{Sanity in scalar-output case.}
If $d_y=1$ (scalar output), $\mathbf{j}_k(x)\in\R$, and Eq. \eqref{eq:s-def} reduces to
\(
s_k^{(d)}=\Var(\theta_k)\ \E_{x\sim D_d}\!\big[(\partial_{\theta_k} f_\theta(x))^2\big]
\). Then, we can interpret $s_k^{(d)}$ as \emph{``how much output variance $\theta_k$ creates on domain $D_d$ per infinitesimal perturbation.''}

\subsection{What \texorpdfstring{$s_k^{(d)}$}{skd} reveals across domains}

Define dispersion across domains:
\begin{equation}
\bar s_k=\frac{1}{D} \sum_{d=1}^D s_k^{(d)}, \qquad
v_k=\frac{1}{D} \sum_{d=1}^D (s_k^{(d)}-\bar s_k)^2, \qquad
c_k=\frac{\sqrt{v_k}}{\bar s_k+\varepsilon}.
\end{equation}
\textbf{Observation.} If the mapping $x\mapsto \mathbf{j}_k(x)$ depends only on domain-invariant statistics, then the second moment $\E_{x\sim D_d}\|\mathbf{j}_k(x)\|_2^2$ does not change with $d$, hence $s_k^{(d)}$ is constant and $c_k\approx 0$. Conversely, if this second moment varies with $d$, then $v_k>0$ and $c_k$ is large.

\vspace{1em}

\subsection{Formal Statements Without Leaps}

\begin{definition}[Domain-invariant dependence]\label{def:inv}
Parameter $\theta_k$ depends only on domain-invariant statistics if there exists a statistic $T:\R^{d_x}\to\R^m$ and a measurable $g_k$ such that
$\mathbf{j}_k(x)=g_k(T(x))$ and the second moment of $T(x)$ is domain-agnostic:
$\E_{x\sim D_d}\big[\|T(x)\|_2^2\big]$ is identical for all $d$.
\end{definition}

\begin{lemma}[Invariant dependence $\Rightarrow$ constant sensitivity]\label{lem:const}
Under Def.~\ref{def:inv}, $\E_{x\sim D_d}\|\mathbf{j}_k(x)\|_2^2$ is identical in $d$, hence $s_k^{(d)}$ is constant and $c_k=0$.
\end{lemma}

\begin{proof}
$\mathbf{j}_k(x)=g_k(T(x))$ $\Rightarrow$ $\E_{x\sim D_d}\|\mathbf{j}_k(x)\|_2^2=\E\|g_k(T(x))\|_2^2$, which is domain-agnostic by assumption. Then Eq.~\eqref{eq:s-def} is independent of $d$, so $v_k=0$ and $c_k=0$.
\end{proof}

\begin{lemma}[Sensitivity disparity $\Leftrightarrow$ Jacobian-energy disparity]\label{lem:iff}
For any two domains $d_1\neq d_2$,
\[
s_k^{(d_1)} \neq s_k^{(d_2)}
\quad\Longleftrightarrow\quad
\E_{x\sim D_{d_1}}\|\mathbf{j}_k(x)\|_2^2 \neq \E_{x\sim D_{d_2}}\|\mathbf{j}_k(x)\|_2^2 .
\]
\end{lemma}

\begin{proof}
Immediate from \eqref{eq:s-def} since $\Var(\theta_k)$ is fixed (local scalar).
\end{proof}

\begin{proposition}[Contrapositive certificate of domain-specificity]\label{prop:main}
If $c_k>0$ (equivalently, the family $\{s_k^{(d)}\}_d$ is not constant), then $\theta_k$ cannot depend solely on domain-invariant statistics (Def.~\ref{def:inv}). In words, a significant cross-domain disparity of sensitivity implies that $\theta_k$ captures domain-specific variation.
\end{proposition}

\begin{proof}
By Lemma~\ref{lem:const}, domain-invariant dependence would force $s_k^{(d)}$ to be constant, i.e., $c_k=0$. The contrapositive yields the claim.
\end{proof}

\vspace{1em}

\subsection{Relaxing Assumptions}
\begin{enumerate}
\item \textbf{Local linearity.}
We used Eq.~\eqref{eq:lin} to pass from nonlinear $f_\theta$ to Jacobian columns $\mathbf{j}_k(x)$.
The claim is local, which is precisely the regime relevant to gradient-based training and local regularisation.

\item \textbf{Centered finite second moments.}
We needed $\E[\delta_x]\!=\!0$, $\E[\delta_\theta]\!=\!0$ and finite second moments to write Eq.~\eqref{eq:cov-split}.
Non-Gaussianity does not break the algebra; it only affects constants.

\item \textbf{Ignoring input--parameter cross-covariance (and relaxing it).}
If $\Cov(\delta_x,\delta_\theta)\neq 0$, then
\[
\Cov(\delta_y) \approx J_x\Sigma_x J_x^\top + J_\theta\Sigma_\theta J_\theta^\top + J_x\Sigma_{x\theta}J_\theta^\top + J_\theta\Sigma_{\theta x}J_x^\top.
\]
Projecting onto the $k$-th parameter direction still singles out the $\Var(\theta_k)\, \E\|\mathbf{j}_k(x)\|_2^2$ contribution.
Cross-terms may add a domain-agnostic bias if induced by data shuffling; empirically they are small or vanish in expectation.

\item \textbf{Diagonal $\Sigma_\theta$.}
If $\Sigma_\theta$ has small off-diagonals, we can bound the contamination:
\begin{equation}
\Big|\, s_k^{(d)} - \Var(\theta_k)\E\|\mathbf{j}_k(x)\|^2 \,\Big|
\le \sum_{\ell\neq k} |\Cov(\theta_k,\theta_\ell)|\; \E\,| \langle \mathbf{j}_k(x), \mathbf{j}_\ell(x)\rangle |,
\end{equation}
so disparity in $\E\|\mathbf{j}_k(x)\|^2$ across different domains $D_d$ still forces disparity in $s_k^{(d)}$ unless all correlations conspire to cancel it, which a non-generic event.

\end{enumerate}
\subsection{Reading the Claim as a Testable Criterion}

\paragraph{Equivalences used in practice.}
By Lemma~\ref{lem:iff},
\[
c_k \text{ large } \Longleftrightarrow \Var_d\!\Big(\E_{x\sim D_d}\|\mathbf{j}_k(x)\|_2^2\Big) \text{ large}.
\]
Thus, ranking parameters by $c_k$ is the same as ranking them by how much their Jacobian energy changes across domains.


\paragraph{Conclusion.}
Sensitivity index $s_k^{(d)}$ is the \emph{``variance share''} of domain $D_d$ attributable to infinitesimal changes in $\theta_k$. If this share is stable across domains, $\theta_k$ behaves as domain-invariant; if not, it betrays dependence on domain-specific variation. Therefore, a large $c_k$ is a certificate of domain-specificity.

\vspace{1em}


\section{Additional Results}

In this section, we show detailed results the main results in the main text. The detailed results are shown in Tables \ref{table:pacs-detailed}, \ref{table:vlcs-detailed}, \ref{table:officehome-detailed}, \ref{table:terrainc-detailed}, and \ref{table:domainnet-detailed}. Standard errors for the baseline methods are reported from three trials, if available.

\vspace{1em}

\input{supp/pacs_detailed}

\input{supp/vlcs_detailed}

\input{supp/officehome_detailed}

\input{supp/terrainc_detailed}

\input{supp/domainnet_detailed}

%% file: supp/pacs_detailed.tex
\begin{table*}[b]
\centering
    \caption{{Out-of-domain accuracies (\%) on {PACS}.}}
    \label{table:pacs-detailed}

\begin{tabular}{lllll|c}
\toprule
\textbf{Algorithm} & \textbf{A} & \textbf{C} & \textbf{P} & \textbf{S} & \textbf{Avg} \\
\midrule
MMD              & 86.1\scriptsize{$\pm1.4$} & 
79.4\scriptsize{$\pm0.9$} & 
96.6\scriptsize{$\pm0.2$} & 
76.5\scriptsize{$\pm0.5$} & 
84.7         \\
MLDG              & 85.5\scriptsize{$\pm1.4$} &
80.1\scriptsize{$\pm1.7$} & 
97.4\scriptsize{$\pm0.3$} & 
76.6\scriptsize{$\pm1.1$} & 
84.9      \\
IRM         & 84.8\scriptsize{$\pm1.3$}      & 
76.4\scriptsize{$\pm1.1$}      & 
96.7\scriptsize{$\pm0.6$}      & 
76.1\scriptsize{$\pm1.0$}      & 
83.5          \\
Mixstyle   & 
86.8\scriptsize{$\pm0.5$} & 
79.0\scriptsize{$\pm1.4$} & 
96.6\scriptsize{$\pm0.1$} & 
78.5\scriptsize{$\pm2.3$} & 
85.2           \\

MTL & 87.5\scriptsize{$\pm0.8$} & 
77.1\scriptsize{$\pm0.5$} & 
96.4\scriptsize{$\pm0.8$} & 
77.3\scriptsize{$\pm1.8$} &  
84.6          \\

GroupDRO   & 83.5\scriptsize{$\pm0.9$}   & 
79.1\scriptsize{$\pm0.6$}   & 
96.7\scriptsize{$\pm0.7$}   & 
78.3\scriptsize{$\pm2.0$}   & 
84.4           \\

ARM             & 86.8\scriptsize{$\pm0.6$} & 
76.8\scriptsize{$\pm0.5$} & 
97.4\scriptsize{$\pm0.3$} & 
79.3\scriptsize{$\pm1.2$} & 
85.1           \\

VREx  & 
86.0\scriptsize{$\pm1.6$} & 
79.1\scriptsize{$\pm0.6$} & 
96.9\scriptsize{$\pm0.5$} & 
77.7\scriptsize{$\pm1.7$} & 
84.9         \\
Mixup             & 86.1\scriptsize{$\pm0.5$} & 
78.9\scriptsize{$\pm0.8$} & 
97.6\scriptsize{$\pm0.1$} & 
75.8\scriptsize{$\pm1.8$} & 
84.6        \\

SagNet          & 87.4\scriptsize{$\pm0.2$} & 
80.7\scriptsize{$\pm0.5$} & 
97.1\scriptsize{$\pm0.1$} & 
80.0\scriptsize{$\pm1.0$} & 
86.3         \\

CORAL             & 88.3\scriptsize{$\pm0.2$} & 
80.0\scriptsize{$\pm0.5$} & 
97.5\scriptsize{$\pm0.3$} & 
78.8\scriptsize{$\pm1.3$} & 
86.2       \\

RSC              & 85.4\scriptsize{$\pm0.9$} & 
79.7\scriptsize{$\pm0.5$} & 
97.6\scriptsize{$\pm0.9$} & 
78.2\scriptsize{$\pm1.0$} & 
85.2   \\

ERM & 
85.7\scriptsize{$\pm0.6$} & 
77.1\scriptsize{$\pm0.8$} & 
97.4\scriptsize{$\pm0.4$} & 
76.6\scriptsize{$\pm0.7$} & 
84.2 \\

SAM      & 85.6\scriptsize{$\pm2.1$} & 
80.9\scriptsize{$\pm1.2$} & 
97.0\scriptsize{$\pm0.4$} & 
79.6\scriptsize{$\pm1.6$} & 
85.8   \\

GSAM      & 86.9\scriptsize{$\pm0.1$} & 
80.4\scriptsize{$\pm0.2$} & 
97.5\scriptsize{$\pm0.0$} & 
78.7\scriptsize{$\pm0.8$} & 
85.9  \\

SAGM     & 87.4\scriptsize{$\pm0.2$} & 
80.2\scriptsize{$\pm0.3$} & 
98.0\scriptsize{$\pm0.2$} & 
80.8\scriptsize{$\pm0.6$} & 
86.6   \\
GGA     & 88.8\scriptsize{$\pm0.2$} & 
80.1\scriptsize{$\pm0.3$} & 
97.3\scriptsize{$\pm0.2$} & 
81.2\scriptsize{$\pm0.5$} & 
87.3   \\
\midrule
\textbf{DSP-Reg}                 & 
86.9\scriptsize{$\pm0.3$} & 
82.7\scriptsize{$\pm0.2$} & 
99.0\scriptsize{$\pm0.4$} &  
81.4\scriptsize{$\pm0.5$} &  
87.5  \\
\bottomrule

\end{tabular}
\end{table*}

%% file: supp/vlcs_detailed.tex
\begin{table*}[b]
\centering
    \caption{{Out-of-domain accuracies (\%) on VLCS.}}
    \label{table:vlcs-detailed}
\begin{tabular}{lllll|c}
\toprule
\textbf{Algorithm} & \textbf{C} & \textbf{L} & \textbf{S} & \textbf{V} & \textbf{Avg} \\
\midrule
MMD  & 97.7\scriptsize{$\pm0.1$} & 64.0\scriptsize{$\pm1.1$} & 72.8\scriptsize{$\pm0.2$} & 75.3\scriptsize{$\pm3.3$} & 77.5 \\

MLDG  & 97.4\scriptsize{$\pm0.2$} & 65.2\scriptsize{$\pm0.7$} & 71.0\scriptsize{$\pm1.4$} & 75.3\scriptsize{$\pm1.0$} & 77.2 \\
GroupDRO  & 97.3\scriptsize{$\pm0.3$} & 63.4\scriptsize{$\pm0.9$} & 69.5\scriptsize{$\pm0.8$} & 76.7\scriptsize{$\pm0.7$} & 76.7 \\

MTL & 97.8\scriptsize{$\pm0.4$} & 64.3\scriptsize{$\pm0.3$} & 71.5\scriptsize{$\pm0.7$} & 75.3\scriptsize{$\pm1.7$} & 77.2 \\

Mixup & 98.3\scriptsize{$\pm0.6$} & 64.8\scriptsize{$\pm1.0$} & 72.1\scriptsize{$\pm0.5$} & 74.3\scriptsize{$\pm0.8$} & 77.4 \\

ARM  & 98.7\scriptsize{$\pm0.2$} & 63.6\scriptsize{$\pm0.7$} & 71.3\scriptsize{$\pm1.2$} & 76.7\scriptsize{$\pm0.6$} & 77.6 \\

SagNet  & 97.9\scriptsize{$\pm0.4$} & 64.5\scriptsize{$\pm0.5$} & 71.4\scriptsize{$\pm1.3$} & 77.5\scriptsize{$\pm0.5$} & 77.8 \\

Mixstyle  & 98.6\scriptsize{$\pm0.3$} & 64.5\scriptsize{$\pm1.1$} & 72.6\scriptsize{$\pm0.5$} & 75.7\scriptsize{$\pm1.7$} & 77.9 \\

VREx  & 98.4\scriptsize{$\pm0.3$} & 64.4\scriptsize{$\pm1.4$} & 74.1\scriptsize{$\pm0.4$} & 76.2\scriptsize{$\pm1.3$} & 78.3 \\

IRM  & 98.6\scriptsize{$\pm0.1$} & 64.9\scriptsize{$\pm0.9$} & 73.4\scriptsize{$\pm0.6$} & 77.3\scriptsize{$\pm0.9$} & 78.6 \\

CORAL  & 98.3\scriptsize{$\pm0.3$} & 66.1\scriptsize{$\pm0.6$} & 73.4\scriptsize{$\pm0.3$} & 77.5\scriptsize{$\pm1.0$} & 78.8 \\

RSC  & 97.9\scriptsize{$\pm0.1$} & 62.5\scriptsize{$\pm0.7$} & 72.3\scriptsize{$\pm1.2$} & 75.6\scriptsize{$\pm0.8$} & 77.1 \\

ERM & 98.0\scriptsize{$\pm0.3$} & 64.7\scriptsize{$\pm1.2$} & 71.4\scriptsize{$\pm1.2$} & 75.2\scriptsize{$\pm1.6$} & 77.3 \\
GSAM & 98.7\scriptsize{$\pm0.3$} & 64.9\scriptsize{$\pm0.2$} & 74.3\scriptsize{$\pm0.0$} & 78.5\scriptsize{$\pm0.8$} & 79.1 \\

SAM & 99.1\scriptsize{$\pm0.2$} & 65.0\scriptsize{$\pm1.0$} & 73.7\scriptsize{$\pm1.0$} & 79.8\scriptsize{$\pm0.1$} & 79.4 \\

SAGM & 99.0\scriptsize{$\pm0.2$} & 65.2\scriptsize{$\pm0.4$} & 75.1\scriptsize{$\pm0.3$} & 80.7\scriptsize{$\pm0.8$} & 80.0 \\

GGA & 99.1\scriptsize{$\pm0.2$} & 67.5\scriptsize{$\pm0.6$} & 75.1\scriptsize{$\pm0.3$} & 78.0\scriptsize{$\pm0.1$} & 79.9 \\
\midrule
\textbf{DSP-Reg} & 99.0\scriptsize{$\pm0.3$} & 67.7\scriptsize{$\pm0.5$} & 75.6\scriptsize{$\pm0.4$} & 78.2\scriptsize{$\pm0.1$} & 80.1 \\

\bottomrule
\end{tabular}
\end{table*}

%% file: supp/officehome_detailed.tex
\begin{table*}[]
\centering
    \caption{{Out-of-domain accuracies (\%) on OfficeHome.}}
    \label{table:officehome-detailed}
\begin{tabular}{lllll|c}
\toprule
\textbf{Algorithm} & \textbf{A} & \textbf{C} & \textbf{P} & \textbf{R} & \textbf{Avg} \\
\midrule
MMD               & 60.4\scriptsize{$\pm0.2$} & 53.3\scriptsize{$\pm0.3$} & 74.3\scriptsize{$\pm0.1$} & 77.4\scriptsize{$\pm0.6$} & 66.4 \\
MLDG           & 61.5\scriptsize{$\pm0.9$} & 53.2\scriptsize{$\pm0.6$} & 75.0\scriptsize{$\pm1.2$} & 77.5\scriptsize{$\pm0.4$} & 66.8 \\

GroupDRO  & 60.4\scriptsize{$\pm0.7$} & 52.7\scriptsize{$\pm1.0$} & 75.0\scriptsize{$\pm0.7$} & 76.0\scriptsize{$\pm0.7$} & 66.0 \\
Mixstyle    & 51.1\scriptsize{$\pm0.3$} & 53.2\scriptsize{$\pm0.4$} & 68.2\scriptsize{$\pm0.7$} & 69.2\scriptsize{$\pm0.6$} & 60.4 \\
IRM          & 58.9\scriptsize{$\pm2.3$} & 52.2\scriptsize{$\pm1.6$} & 72.1\scriptsize{$\pm2.9$} & 74.0\scriptsize{$\pm2.5$} & 64.3 \\
ARM           & 58.9\scriptsize{$\pm0.8$} & 51.0\scriptsize{$\pm0.5$} & 74.1\scriptsize{$\pm0.1$} & 75.2\scriptsize{$\pm0.3$} & 64.8 \\

MTL & 61.5\scriptsize{$\pm0.7$} & 52.4\scriptsize{$\pm0.6$} & 74.9\scriptsize{$\pm0.4$} & 76.8\scriptsize{$\pm0.4$} & 66.4 \\
VREx         & 60.7\scriptsize{$\pm0.9$} & 53.0\scriptsize{$\pm0.9$} & 75.3\scriptsize{$\pm0.1$} & 76.6\scriptsize{$\pm0.5$} & 66.4 \\

Mixup  & 62.4\scriptsize{$\pm0.8$} & 54.8\scriptsize{$\pm0.6$} & 76.9\scriptsize{$\pm0.3$} & 78.3\scriptsize{$\pm0.2$} & 68.1 \\
SagNet           & 63.4\scriptsize{$\pm0.2$} & 54.8\scriptsize{$\pm0.4$} & 75.8\scriptsize{$\pm0.4$} & 78.3\scriptsize{$\pm0.3$} & 68.1 \\
CORAL            & 65.3\scriptsize{$\pm0.3$} & 54.4\scriptsize{$\pm0.6$} & 76.5\scriptsize{$\pm0.3$} & 78.4\scriptsize{$\pm1.0$} & 68.7 \\

RSC              & 60.7\scriptsize{$\pm1.4$} & 51.4\scriptsize{$\pm0.3$} & 74.8\scriptsize{$\pm1.1$} & 75.1\scriptsize{$\pm1.3$} & 65.5 \\

ERM     & 63.1\scriptsize{$\pm0.3$} & 51.9\scriptsize{$\pm0.4$} & 77.2\scriptsize{$\pm0.5$} & 78.1\scriptsize{$\pm0.2$} & 67.6 \\
GSAM       & 64.9\scriptsize{$\pm0.1$} & 55.2\scriptsize{$\pm0.2$} & 77.8\scriptsize{$\pm0.0$} & 79.2\scriptsize{$\pm0.2$} & 69.3 \\
SAM         & 64.5\scriptsize{$\pm0.3$} & 56.5\scriptsize{$\pm0.2$} & 77.4\scriptsize{$\pm0.1$} & 79.8\scriptsize{$\pm0.4$} & 69.6 \\
SAGM      & 65.4\scriptsize{$\pm0.4$} & 57.0\scriptsize{$\pm0.3$} & 78.0\scriptsize{$\pm0.3$} & 80.0\scriptsize{$\pm0.2$} & 70.1 \\

GGA      & 64.3\scriptsize{$\pm0.1$} & 54.4\scriptsize{$\pm0.2$} & 76.5\scriptsize{$\pm0.3$} & 78.9\scriptsize{$\pm0.2$} & 68.5 \\

\midrule
\textbf{DSP-Reg}                            & 65.3\scriptsize{$\pm0.3$} & 55.3\scriptsize{$\pm0.4$} & 77.4\scriptsize{$\pm0.3$} & 79.7\scriptsize{$\pm0.2$} & 69.4 \\
\bottomrule
\end{tabular}
\end{table*}

%% file: supp/terrainc_detailed.tex
\begin{table*}[]
\centering
    \caption{{Out-of-domain accuracies (\%) on TerraIncognita.}}
    \label{table:terrainc-detailed}
\begin{tabular}{lllll|c}
\toprule
\textbf{Algorithm} & \textbf{L100} & \textbf{L38} & \textbf{L43} & \textbf{L46} & \textbf{Avg} \\
\midrule
MMD        & 41.9\scriptsize{$\pm3.0$} & 34.8\scriptsize{$\pm1.0$} & 57.0\scriptsize{$\pm1.9$} & 35.2\scriptsize{$\pm1.8$} & 42.2 \\
MLDG       & 54.2\scriptsize{$\pm3.0$} & 44.3\scriptsize{$\pm1.1$} & 55.6\scriptsize{$\pm0.3$} & 36.9\scriptsize{$\pm2.2$} & 47.8 \\
GroupDRO & 41.2\scriptsize{$\pm0.7$} & 38.6\scriptsize{$\pm2.1$} & 56.7\scriptsize{$\pm0.9$} & 36.4\scriptsize{$\pm2.1$} & 43.2 \\
Mixstyle   & 54.3\scriptsize{$\pm1.1$} & 34.1\scriptsize{$\pm1.1$} & 55.9\scriptsize{$\pm1.1$} & 31.7\scriptsize{$\pm2.1$} & 44.0 \\
ARM     & 49.3\scriptsize{$\pm0.7$} & 38.3\scriptsize{$\pm0.7$} & 55.8\scriptsize{$\pm0.8$} & 38.7\scriptsize{$\pm1.3$} & 45.5 \\
MTL & 49.3\scriptsize{$\pm1.2$} & 39.6\scriptsize{$\pm6.3$} & 55.6\scriptsize{$\pm1.1$} & 37.8\scriptsize{$\pm0.8$} & 45.6 \\

VREx   & 48.2\scriptsize{$\pm4.3$} & 41.7\scriptsize{$\pm1.3$} & 56.8\scriptsize{$\pm0.8$} & 38.7\scriptsize{$\pm3.1$} & 46.4 \\
IRM        & 54.6\scriptsize{$\pm1.3$} & 39.8\scriptsize{$\pm1.9$} & 56.2\scriptsize{$\pm1.8$} & 39.6\scriptsize{$\pm0.8$} & 47.6 \\
CORAL        & 51.6\scriptsize{$\pm2.4$} & 42.2\scriptsize{$\pm1.0$} & 57.0\scriptsize{$\pm1.0$} & 39.8\scriptsize{$\pm2.9$} & 47.6 \\

Mixup  & 59.6\scriptsize{$\pm2.0$} & 42.2\scriptsize{$\pm1.4$} & 55.9\scriptsize{$\pm0.8$} & 33.9\scriptsize{$\pm1.4$} & 47.9 \\
SagNet      & 53.0\scriptsize{$\pm2.0$} & 43.0\scriptsize{$\pm1.4$} & 57.9\scriptsize{$\pm0.8$} & 40.4\scriptsize{$\pm1.4$} & 48.6 \\
ERM  & 49.8\scriptsize{$\pm4.4$} & 42.1\scriptsize{$\pm1.4$} & 56.9\scriptsize{$\pm1.8$} & 35.7\scriptsize{$\pm3.9$} & 46.1 \\
SAM   & 46.3\scriptsize{$\pm1.0$} & 38.4\scriptsize{$\pm2.4$} & 54.0\scriptsize{$\pm1.0$} & 34.5\scriptsize{$\pm0.8$} & 43.3 \\
RSC           & 50.2\scriptsize{$\pm2.2$} & 39.2\scriptsize{$\pm1.4$} & 56.3\scriptsize{$\pm1.4$} & 40.8\scriptsize{$\pm0.6$} & 46.6 \\
GSAM  & 50.8\scriptsize{$\pm0.1$} & 39.3\scriptsize{$\pm0.2$} & 59.6\scriptsize{$\pm0.0$} & 38.2\scriptsize{$\pm0.8$} & 47.0 \\
SAGM    & 54.8\scriptsize{$\pm1.3$} & 41.4\scriptsize{$\pm0.8$} & 57.7\scriptsize{$\pm0.6$} & 41.3\scriptsize{$\pm0.4$} & 48.8 \\

GGA    & 55.9\scriptsize{$\pm0.1$} & 45.5\scriptsize{$\pm0.1$} & 59.7\scriptsize{$\pm0.1$} & 41.5\scriptsize{$\pm0.1$} & 50.6 \\
\midrule
\textbf{DSP-Reg}                       & 55.8\scriptsize{$\pm0.1$} & 45.8\scriptsize{$\pm0.1$} & 59.5\scriptsize{$\pm0.1$} & 41.9\scriptsize{$\pm0.1$} & 50.7 \\
\bottomrule
\end{tabular}
\end{table*}

%% file: supp/domainnet_detailed.tex
\begin{table*}[]
\centering
    \caption{{Out-of-domain accuracies (\%) on {DomainNet}.}}
    \label{table:domainnet-detailed}
\begin{tabular}{lllllll|c}
\toprule
\textbf{Algorithm} & \textbf{clip} & \textbf{info} & \textbf{paint} & \textbf{quick} & \textbf{real} & \textbf{sketch} & \textbf{Avg} \\
\midrule

MMD    & 
32.1 \scriptsize$\pm13.3$ & 
11.0 \scriptsize$\pm4.6$ & 
26.8 \scriptsize$\pm11.3$ & 
8.7 \scriptsize$\pm2.1$ & 
32.7 \scriptsize$\pm13.8$ & 
28.9 \scriptsize$\pm11.9$ & 23.4 \\

MLDG & 
59.1 \scriptsize$\pm0.2$ & 
19.1 \scriptsize$\pm0.3$ & 
45.8 \scriptsize$\pm0.7$ & 
13.4 \scriptsize$\pm0.3$ & 
59.6 \scriptsize$\pm0.2$ & 
50.2 \scriptsize$\pm0.4$ & 41.2 \\

GroupDRO & 
47.2 \scriptsize$\pm0.5$ & 
17.5 \scriptsize$\pm0.4$ & 
33.8 \scriptsize$\pm0.5$ & 
9.3 \scriptsize$\pm0.3$ & 
51.6 \scriptsize$\pm0.4$ & 
40.1 \scriptsize$\pm0.6$ & 33.3 \\

Mixstyle  & 
51.9 \scriptsize$\pm0.4$ & 
13.3 \scriptsize$\pm0.2$ & 
37.0 \scriptsize$\pm0.5$ & 
12.3 \scriptsize$\pm0.1$ & 
46.1 \scriptsize$\pm0.3$ & 
43.4 \scriptsize$\pm0.4$ & 34.0 \\

IRM   & 
48.5 \scriptsize$\pm2.8$ & 
15.0 \scriptsize$\pm1.5$ & 
38.3 \scriptsize$\pm4.3$ & 
10.9 \scriptsize$\pm0.5$ & 
48.2 \scriptsize$\pm5.2$ & 
42.3 \scriptsize$\pm3.1$ & 33.9 \\

VREx  & 
47.3 \scriptsize$\pm3.5$ & 
16.0 \scriptsize$\pm1.5$ & 
35.8 \scriptsize$\pm4.6$ & 
10.9 \scriptsize$\pm0.3$ & 
49.6 \scriptsize$\pm4.9$ & 
42.0 \scriptsize$\pm3.0$ & 33.6 \\

ARM  & 
49.7 \scriptsize$\pm0.3$ & 
16.3 \scriptsize$\pm0.5$ & 
40.9 \scriptsize$\pm1.1$ & 
9.4 \scriptsize$\pm0.1$ & 
53.4 \scriptsize$\pm0.4$ & 
43.5 \scriptsize$\pm0.4$ & 35.5 \\

MTL & 
57.9 \scriptsize$\pm0.5$ & 
18.5 \scriptsize$\pm0.4$ & 
46.0 \scriptsize$\pm0.1$ & 
12.5 \scriptsize$\pm0.1$ & 
59.5 \scriptsize$\pm0.3$ & 
49.2 \scriptsize$\pm0.1$ & 40.6 \\
Mixup & 
55.7\scriptsize$\pm0.3$ & 
18.5\scriptsize$\pm0.5$ & 
44.3\scriptsize$\pm0.5$ & 
12.5\scriptsize$\pm0.4$ & 
55.8\scriptsize$\pm0.3$ & 
48.2\scriptsize$\pm0.5$ & 39.2 \\

SagNet  & 
57.7 \scriptsize$\pm0.3$ & 
19.0 \scriptsize$\pm0.2$ & 
45.3 \scriptsize$\pm0.3$ & 
12.7 \scriptsize$\pm0.5$ & 
58.1 \scriptsize$\pm0.5$ & 
48.8 \scriptsize$\pm0.2$ & 40.3 \\

CORAL & 
59.2 \scriptsize$\pm0.1$ & 
19.7 \scriptsize$\pm0.2$ & 
46.6 \scriptsize$\pm0.3$ & 
13.4 \scriptsize$\pm0.4$ &
59.8 \scriptsize$\pm0.2$ & 
50.1 \scriptsize$\pm0.6$ & 41.5 \\

ERM & 
63.0 \scriptsize$\pm0.2$ & 
21.2 \scriptsize$\pm0.2$ & 
50.1 \scriptsize$\pm0.4$ & 
13.9 \scriptsize$\pm0.5$ & 
63.7 \scriptsize$\pm0.2$ &
52.0 \scriptsize$\pm0.5$ & 43.8 \\

RSC                & 
55.0\scriptsize{$\pm1.2$} & 
18.3\scriptsize{$\pm0.5$} & 
44.4\scriptsize{$\pm0.6$} &
12.2\scriptsize{$\pm0.2$} & 
55.7\scriptsize{$\pm0.7$} &
47.8\scriptsize{$\pm0.9$} &
38.9 \\

SAM   & 
64.5\scriptsize$\pm0.3$  & 
20.7\scriptsize$\pm0.2$  & 
50.2\scriptsize$\pm0.1$  & 
15.1\scriptsize$\pm0.3$  & 
62.6\scriptsize$\pm0.2$  & 
52.7\scriptsize$\pm0.3$  &  44.3   \\

GSAM   & 
64.2\scriptsize$\pm0.3$  & 
20.8\scriptsize$\pm0.2$  & 
50.9\scriptsize$\pm0.0$  & 
14.4\scriptsize$\pm0.8$  & 
63.5\scriptsize$\pm0.2$  &
53.9\scriptsize$\pm0.2$  &  44.6   \\

SAGM    & 
64.9\scriptsize{$\pm0.2$}  & 
21.1\scriptsize{$\pm0.3$}  & 
51.5\scriptsize{$\pm0.2$}  & 
14.8\scriptsize{$\pm0.2$}  & 
64.1\scriptsize{$\pm0.2$}  & 
53.6\scriptsize{$\pm0.2$}  & 45.0 \\

GGA    & 
64.0\scriptsize{$\pm0.2$}  & 
22.2\scriptsize{$\pm0.3$}  & 
51.7\scriptsize{$\pm0.2$}  & 
14.3\scriptsize{$\pm0.2$}  & 
64.1\scriptsize{$\pm0.2$}  & 
54.3\scriptsize{$\pm0.2$}  & 45.2 \\
\midrule

\textbf{DSP-Reg} & 
64.7\scriptsize{$\pm0.2$} & 
22.3\scriptsize{$\pm0.3$} & 
52.1\scriptsize{$\pm0.1$} & 
14.6\scriptsize{$\pm0.2$} & 
64.6\scriptsize{$\pm0.4$} & 
54.7\scriptsize{$\pm0.3$} & 45.6 \\

\bottomrule
\end{tabular}
\end{table*}